\documentclass[11pt]{article}

\usepackage{algorithm}
\usepackage[noend]{algorithmic}

\usepackage{hyperref}
\usepackage{yub}
\usepackage{paralist}

\usepackage{microtype}
\usepackage{graphicx}
\usepackage{subfigure}
\usepackage{booktabs} 
\usepackage{dsfont}
\usepackage{times}
\usepackage{booktabs}
\usepackage{multirow}
\usepackage{makecell}
\usepackage{xspace}

\usepackage[top=1in, bottom=1in, left=1in, right=1in]{geometry}
\usepackage[compress,round]{natbib}

\newcommand{\mat}[1]{\ensuremath{\mathbf{#1}}}

\newcommand{\trans}{^{\top}}
\newcommand{\poly}{\mathrm{poly}}

\newcommand{\cO}{\mathcal{O}}
\newcommand{\tlO}{\mathcal{\tilde{O}}}

\newcommand{\cD}{\mathcal{D}}

\newcommand{\cH}{\mathcal{H}}

\newcommand{\cS}{\mathcal{S}}
\newcommand{\cA}{\mathcal{A}}
\newcommand{\cB}{\mathcal{B}}

\newcommand{\ind}[3]{{#1}^{#2}_{#3}} 
\newcommand{\sind}[3]{{#1}^{#2}_{#3}} 

\def\shownotes{1}  
\ifnum\shownotes=1
\newcommand{\authnote}[2]{$\ll$\textsf{\footnotesize #1 notes: #2}$\gg$}
\else
\newcommand{\authnote}[2]{}
\fi

\renewcommand{\setto}{\leftarrow}

\newcommand{\Reg}{{\rm Regret}}
\newcommand{\WeakReg}{{\rm WeakRegret}}

\newcommand{\up}{{\sf up}}
\newcommand{\low}{{\sf low}}
\newcommand{\mdp}{{\rm MG}}
\newcommand{\nashgeneral}{\textsc{Nash\_General\_Sum}}
\newcommand{\nashzero}{\textsc{Nash\_Zero\_Sum}}
\newcommand{\explore}{\textsc{Reward\_Free\_Exploration}}
\newcommand{\Unif}{{\rm Unif}}
\newcommand{\EULER}{\textsc{EULER}}

\renewcommand{\cite}{\citep}

\makeatletter
\def\blfootnote{\gdef\@thefnmark{}\@footnotetext}
\makeatother

\usepackage{amsthm}

\makeatletter
\newtheorem*{rep@theorem}{\rep@title}
\newcommand{\newreptheorem}[2]{%
\newenvironment{rep#1}[1]{%
 \def\rep@title{#2 \ref{##1}}%
 \begin{rep@theorem}}%
 {\end{rep@theorem}}}
\makeatother

\newreptheorem{theorem}{Theorem}

\hypersetup{
    colorlinks,
    linkcolor={blue!50!black},
    citecolor={blue!50!black},
}
\colorlet{linkequation}{blue}

\title{\bf{Provable Self-Play Algorithms for Competitive Reinforcement Learning}}

\author{
  Yu Bai\thanks{Salesforce Research. \texttt{yu.bai@salesforce.com}}
  \and
  Chi Jin\thanks{Princeton University. \texttt{chij@princeton.edu}}
}
\date{\today}

\begin{document}
\maketitle

\renewcommand{\sec}{Sections_arXiv}


\begin{abstract}
  Self-play, where the algorithm learns by playing against itself without requiring any direct supervision, has become the new weapon in modern Reinforcement Learning (RL) for achieving superhuman performance in practice. However, the majority of exisiting theory in reinforcement learning only applies to the setting where the agent plays against a fixed environment; it remains largely open whether self-play algorithms can be provably effective, especially when it is necessary to manage the exploration/exploitation tradeoff. We study self-play in competitive reinforcement learning under the setting of Markov games, a generalization of Markov decision processes to the two-player case. We introduce a self-play algorithm---Value Iteration with Upper/Lower Confidence Bound (VI-ULCB)---and show that it achieves regret $\tlO(\sqrt{T})$ after playing $T$ steps of the game, where the regret is measured by the agent's performance against a \emph{fully adversarial} opponent who can exploit the agent's strategy at \emph{any} step. We also introduce an explore-then-exploit style algorithm, which achieves a slightly worse regret of $\tlO(T^{2/3})$, but is guaranteed to run in polynomial time even in the worst case. To the best of our knowledge, our work presents the first line of provably sample-efficient self-play algorithms for competitive reinforcement learning.

\end{abstract}
\section{Introduction}
This paper studies \emph{competitive reinforcement learning}
(competitive RL), that is, reinforcement learning with two or more
agents taking actions simultaneously, but each maximizing
their own reward.
Competitive RL is a major branch of the more general
setting of multi-agent reinforcement learning (MARL), with the
specification that the agents have conflicting rewards (so that they
essentially compete with each other) yet can be trained in a
centralized fashion (i.e. each agent has access to the other agents'
policies)~ ~\cite{crandall2005learning}.

There are substantial recent progresses in competitive RL, in
particular in solving hard multi-player games such as
GO~\citep{silver2017mastering},
Starcraft~\citep{vinyals2019grandmaster}, and Dota
2~\citep{openaidota}.  A key highlight in their approaches is the
successful use of \emph{self-play} for achieving super-human
performance in absence of human knowledge or expert opponents. These
self-play algorithms are able to learn a good policy for all players
from scratch through repeatedly playing the current policies against
each other and performing policy updates using these self-played game
trajectories. The empirical success of self-play has challenged the
conventional wisdom that expert opponents are necessary for achieving
good performance, and calls for a better theoretical understanding.

In this paper, we take initial steps towards understanding the
effectiveness of self-play algorithms in competitive RL from a
theoretical perspective. We focus on the special case of two-player
zero-sum Markov games~\citep{shapley1953stochastic,littman1994markov},
a generalization of Markov Decision Processes (MDPs) to the two-player
setting. In a Markov game, the two players share states,
play actions simultaneously, and observe the same reward. However, one player
aims to maximize the return while the other aims to minimize it.  This
setting covers the majority of two-player games including GO (there is
a single reward of $\set{+1, -1}$ at the end of the game indicating
which player has won), and also generalizes zero-sum matrix
games~\citep{neumann1928theorie}---an important game-theoretic
problem---into the multi-step (RL) case. 

More concretely, the goal of this paper is to design low-regret
algorithms for solving episodic two-player Markov games in the
general setting~\citep{kearns2002near}, that is, the
algorithm is allowed to play the game for a fixed amount of episodes
using arbitrary policies, and its performance is measured in terms of
the regret.  We consider a strong notion of regret for two-player
zero-sum games, where the performance of the deployed policies in each
episode is measured against the best response \emph{for that
  policy}, which can be different in different episodes. Such a regret
bound measures the algorithm's ability in managing the exploration and
exploitation tradeoff against fully adaptive opponents, and can
directly translate to other types of guarantees such as the PAC sample
complexity bound. 






\begin{table*}[!t]
  \renewcommand{\arraystretch}{1.6}
  \label{table:rate}
  \centering
  \caption{Regret and PAC guarantees of the Algorithms in this paper for zero-sum Markov games.}

  \label{tab:results_cc} \small
  
  \begin{tabular}{|>{\centering}m{2cm}|>{\centering}m{4.2cm}|c|c|c|} \hline
 \textbf{Settings} & \textbf{Algorithm} & \textbf{Regret} & \textbf{PAC} & \textbf{Runtime}\\ \hline 
 \multirow{3}{*}{\shortstack{General\\Markov Game}} & VI-ULCB (Theorem \ref{thm:main_sim}) & $\tlO(\sqrt{H^3 S^2 ABT})$ & $\tlO(H^4 S^2 AB/\epsilon^2)$ & PPAD-complete\\ \cline{2-5}
 & VI-explore (Theorem \ref{thm:PAC_explore}) & $\tlO((H^5S^2ABT^2)^{1/3})$ & $\tlO(H^5S^2AB/\epsilon^2)$ & \multirow{4}{*}{Polynomial} \\ \cline{2-4}
 & Mirror Descent ($H=1$) \citep{rakhlin2013optimization}  & $\tlO(\sqrt{S(A+B)T})$ & $\tlO(S(A+B)/\epsilon^2)$ & \\ \cline{1-4}
 \multirow{2}{*}{\shortstack{Turn-Based\\Markov Game}} & VI-ULCB (Corollary \ref{cor:main_turn}) & $\tlO(\sqrt{H^3S^2(A+B)T})$ & $\tlO(H^4S^2(A+B)/\epsilon^2)$ & \\ \cline{2-4}
 & Mirror Descent ($H=2$) (Theorem \ref{theorem:turn-based-mirror-descent})   & $\tlO(\sqrt{S(A+B)T})$ & $\tlO(S(A+B)/\epsilon^2)$ & \\ \hline
 Both & Lower Bound (Corollary \ref{cor:lowerbound}) & $\Omega(\sqrt{H^2S(A+B)T})$ & $\Omega(H^2S(A+B)/\epsilon^2)$ & - \\ \hline
\end{tabular}
\end{table*}

\paragraph{Our contribution}
This paper introduces the first line of provably sample-efficient
self-play algorithms for zero-sum Markov game under no restrictive
assumptions. Concretely, 
\begin{itemize}
\item We introduce the first self-play algorithm with $\tlO(\sqrt{T})$
  regret for zero-sum Markov games. More specifically, it achieves
  $\tlO(\sqrt{H^3S^2ABT})$ regret in the general case, where $H$ is
  the length of the game, $S$ is the number of states, $A, B$ are the
  number of actions for each player, and $T$ is the total number of
  steps played. In special case of turn-based games, it achieves
  $\tlO(\sqrt{H^3S^2(A+B)T})$ regret with guaranteed polynomial
  runtime.
\item We also introduce an explore-then-exploit style algorithm. It has guaranteed polynomial runtime in the general setting of zero-sum Markov games, with a slightly worse $\tlO(T^{2/3})$ regret.

\item We raise the open question about the optimal dependency of the
  regret on $S, A, B$. We provide a lower bound
  $\Omega(\sqrt{S(A+B)T})$, and show that the lower bound can be
  achieved in simple case of two-step turn-based games by a mirror
  descent style algorithm. 

\end{itemize}
Above results are summarized in Table \ref{table:rate}.



\subsection{Related Work}

There is a fast-growing body of work on multi-agent reinforcement learning (MARL). Many of them achieve striking empirical performance, or attack MARL in the cooperative setting, where agents are optimizing for a shared or similar reward. We refer the readers to several recent surveys for these results \citep[see e.g.][]{bucsoniu2010multi,nguyen2018deep, oroojlooyjadid2019review, zhang2019multi}. In the rest of this section we focus on theoretical results related to competitive RL.

\paragraph{Markov games}
Markov games (or stochastic games) is proposed as a mathematical model for compeitive RL back in the early 1950s~\cite{shapley1953stochastic}.
There is a long line of classical work since then on solving this problem \citep[see e.g.][]{littman1994markov, littman2001friend, hu2003nash,hansen2013strategy}. They design algorithms, possibly with runtime guarantees, to find optimal policies in Markov games when both the transition matrix and reward are known, or in the asymptotic setting where number of data goes to infinity. These results do not directly apply to the non-asymptotic setting where the transition and reward are unknown and only a limited amount of data are available for estimating them.

A few recent work tackles self-play algorithms for Markov games in the non-asymptotic setting, working under either structural assumptions about the game or stronger sampling oracles.
\citet{wei2017online} propose an upper confidence algorithm for stochastic games and prove that a self-play style algorithm finds $\epsilon$-optimal policies in $\poly(1/\epsilon)$ samples. \citet{jia2019feature, sidford2019solving} study turn-based stochastic games---a special case of general Markov games, and propose algorithms with near-optimal sample complexity. However, both lines of work make strong assumptions---on either the structure of Markov games or how we access data---that are not always true in practice. Specifically, \citet{wei2017online} assumes no matter what strategy one agent sticks to, the other agent can always reach all states by playing a certain policy, and \citet{jia2019feature, sidford2019solving} assume access to simulators (or generative models) which enable the agent to directly sample transition and reward information for any state-action pair. These assumptions greatly alleviate the challenge in exploration. In contrast, our results apply to general Markov games without further structural assumptions, and our algorithms have built-in mechanisms for solving the challenge in the exploration-exploitation tradeoff. 

Finally, we note that classical \textsc{R-max} algorithm \citep{brafman2002r} does not make restrictive assumptions. It also has provable guarantees even when playing against the adversarial opponent in Markov game. However, the theoretical guarantee in \citet{brafman2002r} is weaker than the standard regret, and does not directly imply any self-play algorithm with regret bound in our setting (See Section \ref{app:connection} for more details).

\paragraph{Adversarial MDP}
Another line of related work focuses on provable algorithms against \emph{adversarial opponents} in MDP. Most work in this line considers the setting with adversarial rewards \citep[see e.g.][]{zimin2013online, rosenberg2019online, jin2019learning}. These results do not direcly imply provable self-play algorithms in our setting, because the adversarial opponent in Markov games can affect both the reward and the transition. There exist a few works that tackle both adversarial transition functions and adversarial rewards~\citep{yu2009arbitrarily, cheung2019reinforcement, lykouris2019corruption}. In particular, \citet{lykouris2019corruption} considers a stochastic problem with $C$ episodes arbitrarily corrupted and obtain $\cO(C\sqrt{T}+C^2)$ regret. When applying these results to Markov games with an adversarial opponent, $C$ can be $\Theta(T)$ without further assumptions, which makes the bound vacuous.


\paragraph{Single-agent RL}
There is an extensive body of research on the sample efficiency of reinforcement learning in the single agent setting \citep[see e.g.][]{jaksch2010near, osband2014generalization, azar2017minimax, dann2017unifying, strehl2006pac, jin2018q}, which are studied under the model of Markov decision process---a special case of Markov games. For the tabular episodic setting with nonstationary dynamics and no simulators, the best regrets achieved by existing model-based and model-free algorithms are $\tilde{\mathcal{O}}(\sqrt{H^2 S A T})$ \cite{azar2017minimax} and $\tilde{\mathcal{O}}(\sqrt{H^3 S A T})$ \cite{jin2018q}, respectively, where $S$ is the number of states, $A$ is the number of actions, $H$ is the length of each episode, and $T$ is the total number of steps played. Both of them (nearly) match the minimax lower bound $\Omega(\sqrt{H^2 S A T})$~\cite{jaksch2010near, osband2016lower, jin2018q}.

\section{Preliminaries} \label{sec:prelim}

In this paper, we consider zero-sum Markov
Games (MG)~\citep{shapley1953stochastic, littman1994markov}, which also known as stochastic games in the literature. Zero-sum Markov games are generalization of standard Markov Decision Processes (MDP) into the two-player setting, in which the \emph{max-player} seeks to maximize the total return and the \emph{min-player} seeks to minimize the total return. 


Formally, we consider tabular episodic zero-sum Markov games of the
form $\mdp(H, \cS, \cA, \cB, \P, r)$, where
\begin{itemize}
\item $H$ is the number of steps in each episode.
 \item $\cS = \cup_{h\in[H+1]}\cS_h$, and $\cS_h$ is the set of states at step $h$, with $\max_{h\in[H+1]}|\cS_h| \le S$.
 \item $\cA = \cup_{h\in[H]}\cA_h$, and $\cA_h$ is the set of actions of the max-player at step $h$,
   with $\max_{h\in[H]}|\cA_h| \le A$.
 \item $\cB = \cup_{h\in[H]}\cB_h$, and $\cB_h$ is the set of actions of the min-player at step $h$,
   with $\max_{h\in[H]}|\cB_h| \le B$.
 \item $\P = \{\P_h\}_{h\in[H]}$ is a collection of transition matrices, so that $\P_h ( \cdot | s, a, b) $
   gives the distribution over states if action pair $(a, b)$ is taken for state $s$ at step $h$.
 \item $r = \{r_h\}_{h\in[H]}$ is a collection of reward functions, and $r_h \colon \cS_h \times \cA_h \times \cB_h \to [0,1]$ is the
   deterministic reward function at step $h$. Note that we are assuming that rewards are in $[0,1]$ for normalization. 
   \footnote{While we study deterministic reward functions for notational simplicity, our results generalize to randomized reward functions.}  
\end{itemize}


In each episode of this MG, an initial state $s_1$ is picked
arbitrarily by an adversary.  Then, at each step $h \in [H]$, both
players observe state $s_h \in \cS_h$, pick the action
$a_h \in \cA_h, b_h \in \cB_h$ simultaneously, receive reward
$r_h(s_h, a_h, b_h)$, and then transition to the next state
$s_{h+1}\sim\P_h(\cdot | s_h, a_h, b_h)$. The episode ends when
$s_{H+1}$ is reached.

\paragraph{Policy and value function}
A policy $\mu$ of the max-player is a collection of $H$ functions
$\big\{ \mu_h: \cS \rightarrow \Delta_{\cA_h} \big\}_{h\in [H]}$, where
$\Delta_{\cA_h}$ is the probability simplex over action set
$\cA_h$. Similarly, a policy $\nu$ of the min-player is a collection of
$H$ functions
$\big\{ \nu_h: \cS \rightarrow \Delta_{\cB_h} \big\}_{h\in [H]}$. We use
the notation $\mu_h(a|s)$ and $\nu_h(b|s)$ to present the probability
of taking action $a$ or $b$ for state $s$ at step $h$ under policy $\mu$ or $\nu$
respectively.  We use
$\sind{V}{\mu, \nu}{h} \colon \cS_h \to \mathbb{R}$ to denote the value
function at step $h$ under policy $\mu$ and $\nu$, so that
$\sind{V}{\mu, \nu}{h}(s)$ gives the expected cumulative rewards
received under policy $\mu$ and $\nu$, starting from $s_h = s$, until the end of
the episode:
\begin{equation*}
  \sind{V}{\mu, \nu}{h}(s) \defeq \E_{\mu, \nu}\left[\left.\sum_{h' =
        h}^H r_{h'}(s_{h'}, a_{h'}, b_{h'}) \right| s_h = s\right]. 
\end{equation*}
We also define
$\ind{Q}{\mu, \nu}{h}:\cS_h \times \cA_h \times \cB_h \to \mathbb{R}$
to denote $Q$-value function at step $h$ so that
$\sind{Q}{\mu, \nu}{h}(s, a)$ gives the cumulative
rewards received under policy $\mu$ and $\nu$, starting from
$s_h = s, a_h = a, b_h = b$, till the end of the episode:
\begin{align*}
  \sind{Q}{\mu, \nu}{h}(s, a, b) \defeq  \E_{\mu,
    \nu}\left[\left.\sum_{h' = h}^H r_{h'}(s_{h'},  a_{h'}, b_{h'})
    \right| s_h = s, a_h = a, b_h = b\right]. 
\end{align*}
For simplicity, we use notation of operator $\P_h$ so that
$[\P_h V](s, a, b) \defeq \E_{s' \sim \P_h(\cdot|s, a,
  b)}V(s')$ for any value function $V$. By definition of value functions, for all
$(s, a, b, h) \in \cS_h \times \cA_h \times \cB_h \times [H]$, we have the Bellman
equation
\begin{align}
  \sind{Q}{\mu, \nu}{h}(s, a, b) =
  & ~(r_h + \P_h \sind{V}{\mu, \nu}{h+1})(s, a, b), \label{eq:bellman1}
  \\ 
  \sind{V}{\mu, \nu}{h}(s)
  = & ~\sum_{a, b} \mu_h(a|s) \nu_h(b|s) \sind{Q}{\mu, \nu}{h}(s, a,
      b). \label{eq:bellman2}
\end{align}
where we define $V^{\mu, \nu}_{H+1}(s) = 0$ for all $s \in \cS_{H+1}$

\paragraph{Best response and regret}
We now define the notion of best response and review some
basic properties of it (cf.~\citep{filar2012competitive}), which will
motivate our definition of the regret in two-player Markov games. For any
max-player strategy $\mu$, there exists a \emph{best response} of the min-player, which is a policy
$\nu^\dagger(\mu)$ satisfying
$V_h^{\mu, \nu^\dagger(\mu)}(s) = \inf_{\nu} V_h^{\mu, \nu}(s)$ for
any $(s, h)$. For simplicity, we denote
$V_h^{\mu, \dagger} \defeq V_h^{\mu, \nu^\dagger(\mu)}$. By symmetry, we
can define the best response of the max-player $\mu^\dagger(\nu)$, and define $V_h^{\dagger, \nu}$.  The value functions $V_h^{\mu, \dagger}$
and $V_h^{\dagger, \nu}$ satisfy the following Bellman optimality
equations:
\begin{align}
  \sind{V}{\mu, \dagger}{h}(s) =
  & \inf_{\nu \in \Delta_{\cB_h}}\sum_{a,b} \mu_h(a|s) \nu(b) \sind{Q}{\mu,
    \dagger}{h}(s, a, b), \\ 
  \sind{V}{\dagger, \nu}{h}(s) =
  & \sup_{\mu \in \Delta_{\cA_h}}\sum_{a, b} \mu(a) \nu_h(b|s)
    \sind{Q}{\dagger, \nu}{h}(s, a, b).
\end{align}
It is further known that there exist policies $\mu^\star$, $\nu^\star$
that are optimal against the best responses of the opponent:
\begin{equation*}
  \begin{aligned}
    & V^{\mu^\star, \dagger}_h(s) = \sup_{\mu}
    V^{\mu, \dagger}_h(s), \\
    &  V^{\dagger, \nu^\star}_h(s) = \inf_{\nu}
    V^{\dagger, \nu}_h(s),
  \end{aligned}
  \qquad \textrm{for all}~(s, h).
\end{equation*}
It is also known that, for any $(s, h)$, the minimax
theorem holds: 
\begin{equation*}
\sup_{\mu} \inf_{\nu} V^{\mu, \nu}_h(s) = V^{\mu^\star, \nu^\star}_h(s) = \inf_{\nu} \sup_{\mu} V^{\mu, \nu}_h(s).
\end{equation*}
Therefore, the optimal strategies $(\mu^\star,\nu^\star)$ are also
the Nash Equilibrium for the Markov game. Based on this, it is sensible to
measure the suboptimality of any pair of policies $(\hat{\mu},
\hat{\nu})$ using the gap between their performance and the
performance of the optimal strategy when playing against the best
responses respectively, i.e.,
\begin{equation}
  \label{equation:regret-decompose-1}
  \begin{aligned}
    \brac{V^{\dagger, \hat{\nu}}_h(s) - \inf_{\nu} V^{\dagger,
        \nu}_h(s)} + \brac{\sup_{\mu} V^{\mu, \dagger}_h(s) -
      V^{\hat{\mu},\dagger}_h(s)}
    = V^{\dagger, \hat{\nu}}_h(s) - V^{\hat{\mu}, \dagger}_h(s).  
  \end{aligned}
\end{equation}
We make this formal in the following definition of the regret.
\begin{definition}[Regret]
  For any algorithm that plays the Markov game for $K$ episodes
  with (potentially adversarial) starting state $s_1^k$ for each episode
  $k=1,2,\dots,K$, the regret is defined as
  \begin{equation*}
    \Reg(K) = \sum_{k=1}^K \left[\sind{V}{\dagger, \nu^k}{1}
      (\ind{s}{k}{1}) - \sind{V}{\mu^k, \dagger}{1} (\ind{s}{k}{1})\right],
  \end{equation*}
  where $(\mu^k$, $\nu^k)$ denote the policies deployed by the algorithm
  in the $k$-th episode.
\end{definition}

We note that as a unique feature of self-play algorithms, the learner is playing against herself, and thus chooses strategies for both max-player and min-player at each episode.



\subsection{Turn-based games}
\label{section:turn-based}
In zero-sum Markov games, each step involves the two players playing
simultaneously and independently. It is a general framework, which contains a very important special case---\emph{turn-based} games.~\citep{shapley1953stochastic,jia2019feature}.


The main feature of a turn-based game is that
only one player is taking actions in each step; in other words, 
the max and min player take turns to play the game. Formally, a
turn-based game can be defined through a partition of steps
$[H]$ into two sets $\cH_{\max}$ and $\cH_{\min}$, where $\cH_{\max}$ and $\cH_{\min}$ denote
the sets of steps the max-player and the min-player choose the actions respectively,
which satisfies $\cH_{\max} \cap \cH_{\min} = \emptyset$ and
$\cH_{\max} \cup \cH_{\min} = [H]$. As a special example, GO is a
turn-based game in which the two players play in alternate turns, i.e.
\begin{equation*}
  \mc{H}_{\max} = \set{1, 3, \dots, H-1}~~~{\rm and}~~~\mc{H}_{\min} = \set{2,
    4, \dots,H}
\end{equation*}

Mathematically, we can specialize general zero-sum Markov games to
turn-based games by restricting $\cA_h = \{\mathring{a}\}$ for all $h
\in \mc{H}_{\min}$, and $\cB_h = \{\mathring{b}\}$ for all $h \in
\mc{H}_{\max}$, where $\mathring{a}$ and $\mathring{b}$ are special
dummy actions. Consequently,
in those steps, $\cA_h$ or $\cB_h$ has only a single action as its
element, i.e. the corresponding player can not affect the game in
those steps. A consequence of this specialization
is that
the Nash Equilibria for turn-based games are pure strategies
(i.e. deterministic policies)~\citep{shapley1953stochastic}, similar
as in one-player MDPs.
This is not always true for general Markov games.
\section{Main Results}
\label{section:main}
In this section, we present our algorithm and main theorems. In particular, our algorithm is the first self-play algorithm
that achieves $\tlO(\sqrt{T})$ regret in Markov Games. We describe the algorithm in Section \ref{sec:main_alg}, and present its theoretical guarantee for general Markov games in Section \ref{sec:main_mg}. In Section \ref{sec:main_tbg}, we show that when specialized to turn-based games, the regret and runtime of our algorithm can be further improved.


\subsection{Algorithm description} \label{sec:main_alg}
To solve
zero-sum Markov games, the main idea is to extend the celebrated UCB (Upper
Confidence Bounds) principle---an algorithmic principle that achieves
provably efficient exploration in bandits~\citep{auer2002finite} and
single-agent RL~\citep{azar2017minimax,jin2018q}---to the two-player
setting. Recall that in single-agent RL, the provably efficient {\tt UCBVI}
algorithm~\citep{azar2017minimax} proceeds as
\begin{quote}
  {\bf Algorithm} ({\tt UCBVI} for single-player RL): Compute
  $\set{Q^\up_h(s,a):h,s,a}$ based on estimated transition and
  optimistic (upper) estimate of reward, then play one episode with
  the greedy policy with respect to $Q^\up$.
\end{quote}
Regret bounds for {\tt UCBVI} is then established by showing and
utilizing the fact that $Q^\up$ remains an optimistic (upper) estimate
of the optimal $Q^\star$ throughout execution of the algorithm.

In zero-sum games, the two player have \emph{conflicting} goals: the
max-player seeks to maximize the return and the min-player seeks to
minimize the return. Therefore, it seems natural here to maintain two
sets of Q estimates, one upper bounding the true value and one lower
bounding the true value, so that each player can play optimistically
with respect to her own goal. We summarize this idea into the
following proposal.
\begin{quote}
  {\bf Proposal} (Naive two-player extension of {\tt UCBVI}): Compute
  $\set{Q^\up_h(s,a,b),Q^\low_h(s,a,b)}$ based on estimated transition
  and \{upper, lower\} estimates of rewards, then play one episode
  where the max-player ($\mu$) is greedy with respect to $Q^\up$ and the min-player ($\nu$) is greedy with respect to $Q^\low$.
\end{quote}
However, the above proposal is not yet a well-defined algorithm: 
a greedy strategy $\mu$ with respect to $Q^\up$
requires the knowledge
of how the other player chooses $b$, and vice versa. Therefore, what
we really want is not that ``$\mu$ is greedy with respect to
$Q^\up$'', but rather that ``$\mu$ is greedy with respect to $Q^\up$
when the other player uses $\nu$'', and vice versa. In other words, we
rather desire that $(\mu, \nu)$ are \emph{jointly greedy} with
respect to $(Q^\up, Q^\low)$.




Our algorithm concretizes such joint greediness precisely, building on
insights from one-step matrix games: we choose $(\mu_h, \nu_h)$ to be
the Nash equilibrium for the \emph{general-sum game} in which the
payoff matrix for the max player is $Q^\up$ and for the min player is
$Q^\low$. In other words, both player have their own payoff matrix
(and they are not equal), but they jointly determine their
policies. Formally, we let $(\mu, \nu)$ be determined as
\begin{align*}
   (\mu_h(\cdot|s), \nu_h(\cdot|s)) 
  = \nashgeneral(Q^\up_h(s,\cdot,\cdot), Q^\low_h(s,\cdot,\cdot))
\end{align*}
for all $(h,s)$, where \nashgeneral~is a subroutine that takes two
matrices $\mat{P}, \mat{Q} \in \R^{A\times B}$, and returns the Nash
equilibrium $(\phi, \psi) \in \Delta_A \times \Delta_B$ for general
sum game, which satisfies
\begin{align}
  \label{equation:nash-general-sum}
  \phi\trans \mat{P} \psi = \max_{\tilde{\phi}} \tilde{\phi}\trans
  \mat{P} \psi, \quad \phi\trans \mat{Q} \psi =
  \min_{\tilde{\psi}} 
  \phi\trans \mat{Q} \tilde{\psi}.
\end{align}
Such an equilibrium is guaranteed to exist due to the seminal work
of~\citet{nash1951non}, and is computable by algorithms such as the Lemke-Howson
algorithm~\citep{lemke1964equilibrium}. With the \nashgeneral~subroutine in hand, our
algorithm can be briefly described as
\begin{quote}
  {\bf Our Algorithm} (VI-ULCB): Compute
  $\set{Q^\up_h(s,a,b),Q^\low_h(s,a,b)}$ based on estimated transition
  and \{upper, lower\} estimates of rewards, along the way determining
  policies $(\mu,\nu)$ by running the \nashgeneral~subroutine on
  $(Q^\up,Q^\low)$. Play one episode according to $(\mu,\nu)$.
\end{quote}
The full algorithm is described in
Algorithm~\ref{algorithm:ucbvi_sim}.


\begin{algorithm}[t]
   \caption{Value Iteration with Upper-Lower Confidence Bound  (VI-ULCB)}
   \label{algorithm:ucbvi_sim}
\begin{algorithmic}[1]
   \STATE {\bfseries Initialize:} for any $(s, a, b, h)$,
   $Q_{h}^\up(s,a, b)\setto H$, $Q_{h}^\low(s,a, b)\setto 0$,
   $N_{h}(s,a, b)\setto 0$, $N_h(s,a,b,s')\setto 0$.
   \FOR{episode $k=1,\dots,K$}
   \FOR{step $h=H,H-1,\dots,1$}
   \FOR{$(s, a, b)\in\cS_h\times\cA_h\times \cB_h$}
   \STATE $t=N_{h}(s, a, b)$;
   \STATE $Q_{h}^\up(s, a, b)\setto\min\{\hat{r}_h(s, a, b) +
   [\widehat{\P}_{h} V_{h+1}^\up](s, a, b) +
   \beta_t, H\}$
   \STATE $Q_{h}^\low(s, a, b)\setto\max\{\hat{r}_h(s, a, b) +
   [\widehat{\P}_{h} V_{h+1}^\low](s, a, b)
   - \beta_t, 0\}$
   \ENDFOR
   \FOR{$s \in \cS_h$}
   \STATE $(\mu_h(\cdot|s), \nu_h(\cdot|s))\setto\textsc{Nash\_General\_Sum}(Q_h^\up(s, \cdot, \cdot), Q_h^\low(s, \cdot, \cdot))$
   \STATE $V_h^\up(s) \leftarrow \sum_{a, b}\mu_h(a|s)\nu_h(b|s) Q_h^\up(s, a, b) $.
   \STATE $V_h^\low(s) \leftarrow \sum_{a, b}\mu_h(a|s)\nu_h(b|s) Q_h^\low(s, a, b)$.
   \ENDFOR
   \ENDFOR
   \STATE Receive $s_1$.
   \FOR{step $h=1,\dots, H$}
   \STATE Take action $a_h \sim \mu_h(s_h)$, $b_h \sim \nu_h(s_h)$.
   \STATE Observe reward $r_h$ and next state
   $s_{h+1}$. 
   \STATE $N_{h}(s_h, a_h, b_h)\setto N_{h}(s_h, a_h, b_h) + 1$.
   \STATE $N_h(s_h, a_h, b_h, s_{h+1}) \setto N_h(s_h, a_h, b_h, s_{h+1})
   + 1$
   \STATE $\hat{\P}_h(\cdot|s_h, a_h,
     b_h)\setto \dfrac{N_h(s_h, a_h, b_h, \cdot)}{N_h(s_h, a_h, b_h)}$.
   \STATE $\hat{r}_h(s_h, a_h, b_h)\setto r_h$.
   \ENDFOR
   \ENDFOR
\end{algorithmic}
\end{algorithm}

\subsection{Guarantees for General Markov Games} \label{sec:main_mg}
We are now ready to present our main theorem.

\begin{theorem}[Regret bound for VI-ULCB]
  \label{thm:main_sim}
  For zero-sum Markov games, Algorithm~\ref{algorithm:ucbvi_sim} (with
  choice of bonus $\beta_t=c\sqrt{H^2S\iota/t}$ for large absolute constant $c$) achieves regret
  \begin{align*}
    \Reg(K)  \le \cO\paren{\sqrt{H^3S^2\brac{\max_{h\in[H]}
    A_hB_h}T\iota}} 
    \le \cO\paren{\sqrt{H^3S^2ABT\iota}}
  \end{align*}
  with probability at least $1-p$, where $\iota = \log(SABT/p)$. 
\end{theorem}
We defer the proof of Theorem~\ref{thm:main_sim} into
Appendix~\ref{appendix:proof-main_sim}.

\paragraph{Optimism in the face of uncertainty and best response}
An implication of Theorem~\ref{thm:main_sim} is that a low regret can
be achieved via \emph{self-play}, i.e. the algorithm plays with itself
and does not need an expert as its opponent. This is intriguing
because the regret is measured in terms of the suboptimality against
the worst-case opponent:
\begin{align*}
  & \quad \Reg(K) = \sum_{k=1}^K \brac{V_1^{\dagger, \nu^k}(s_1^k) -
    V_1^{\mu^k,\dagger}(s_1^k)} \\
  & =  \sum_{k=1}^K \underbrace{\brac{\max_{\mu}
    V_1^{\mu, \nu^k}(s_1^k) - V_1^{\mu^k,\nu^k}(s_1^k)}}_{\textrm{gap
    between $\mu^k$ and the best response to $\nu^k$}} + 
    \underbrace{\brac{V_1^{\mu^k,\nu^k}(s_1^k) - \min_\nu
    V_1^{\mu^k,\nu}(s_1^k)}}_{\textrm{gap
    between $\nu^k$ and the best response to $\mu^k$}}. 
\end{align*}
(Note that this decomposition of the regret has a slightly different 
form from~\eqref{equation:regret-decompose-1}.)  Therefore, 
Theorem~\ref{thm:main_sim} demonstrates that self-play can protect
against fully adversarial opponent even when such a strong opponent is not
explicitly available.

The key technical reason enabling such a guarantee is that our $Q$
estimates are optimistic in the face of both the uncertainty of the
game, as well as the best response from the opponent. More precisely, we show that
the $(Q^\up, Q^\low)$ in Algorithm~\ref{algorithm:ucbvi_sim} satisfy
with high probability
\begin{align*}
  Q^{\up,k}_h(s, a, b) \ge \sup_\mu Q^{\mu, \nu^k}_h(s, a, b)
  \ge \inf_{\nu} Q^{\mu^k, \nu}_h(s, a, b) \ge Q^{\low,k}_h(s, a, b)
\end{align*}
for all $(s,a,b,h,k)$, where $(Q^{\up, k}, Q^{\low, k})$ denote the
running $(Q^\up, Q^\low)$ at the beginning of the $k$-th episode
(Lemma~\ref{lem:ULCB_sim}). In constrast, such a guarantee (and
consequently the regret bound) is not achievable if the upper and
lower estimates are only guaranteed to \{upper, lower\} bound the
\emph{values} of the Nash equilibrium.



\paragraph{Translation to PAC bound}
Our regret bound directly implies a PAC sample complexity bound for
learning near-equilibrium policies, based on an online-to-batch
conversion. We state this in the following Corollary, and defer the
proof to Appendix~\ref{appendix:proof-pac}.
\begin{corollary}[PAC bound for VI-ULCB]
  \label{corollary:pac}
  Suppose the initial state of Markov game is fixed at $s_1$, then there exists a pair of (randomized) policies
  $(\what{\mu}, \what{\nu})$ derived through the VI-ULCB algorithm
  such that with probability at least $1-p$ (over the randomness in the
  trajectories) we have 
  \begin{align*}
    \E_{\what{\mu}, \what{\nu}}\brac{V^{\dagger, \what{\nu}}(s_1) -
    V^{\what{\mu}, \dagger}(s_1)} \le \epsilon,
  \end{align*}
  as soon as the number of episodes $K\ge \Omega(H^4S^2AB\iota/\epsilon^2)$, 
  where $\iota = \log(HSAB/(p\epsilon))$, and the expectation is over the randomization in
  $(\what{\mu}, \what{\nu})$.
\end{corollary}


\paragraph{Runtime of Algorithm~\ref{algorithm:ucbvi_sim}}
Algorithm~\ref{algorithm:ucbvi_sim} involves the
\nashgeneral~subroutine for computing the Nash equilibrium of a
general sum matrix game. However, it is known that the computational
complexity for approximating\footnote{More precisely, our proof
  requires the subroutine to find a $(1+1/H)$-multiplicative
  approximation of the equilibrium, that is, for payoff matrices
  $\mat{P},\mat{Q}\in\R^{A\times B}$ we desire vectors
  $\phi\in\Delta_A$ and $\psi\in\Delta_B$ such that
  $\max_{\tilde{\phi}}\tilde{\phi}^\top\mat{P}\psi -
  \min_{\tilde{\psi}} \phi^\top\mat{Q}\tilde{\psi} \le
  (1+1/H)\phi^\top(\mat{P}-\mat{Q})\psi$.  } such an equilibrium is
PPAD-complete~\citep{daskalakis2013complexity}, a complexity class
conjectured to not enjoy polynomial or quasi-polynomial time
algorithms. Therefore, Algorithm~\ref{algorithm:ucbvi_sim} is strictly
speaking not a polynomial time algorithm, despite of being rather sample-efficient.


We note however that there exists practical implementations of the
subroutine such as the Lemke-Howson
algorithm~\citep{lemke1964equilibrium} that can usually find the
solution efficiently. We will further revisit the computational issue
in Section~\ref{section:efficient-algorithm}, in which we design
a computationally efficient algorithm for zero-sum games with
a slightly worse $\tlO(T^{2/3})$ regret.


\subsection{Guarantees for Turn-based Markov Games} \label{sec:main_tbg}
We now instantiate Theorem~\ref{thm:main_sim} on turn-based games
(introduced in Section~\ref{section:turn-based}), in
which the same algorithm enjoys  better regret guarantee and
polynomial runtime. Recall that in turn-based games, for all $h$, we
have either $A_h=1$ or $B_h=1$, therefore given $\max_h A_h\le A$
and $\max_h B_h\le B$ we have
\begin{equation*}
  \max_h A_hB_h \le \max\set{A, B} \le A+B,
\end{equation*}
and thus by Theorem~\ref{thm:main_sim} the regret of
Algorithm~\ref{algorithm:ucbvi_sim} on turn-based games is
bounded by $\cO(\sqrt{H^3S^2(A+B)T})$.

Further, since either $A_h=1$ or $B_h=1$, all the
\nashgeneral~subroutines reduce to \emph{vector} games rather than
matrix games, and can be trivially implemented in polynomial (indeed
linear) time. Indeed, suppose the payoff matrices
in~\eqref{equation:nash-general-sum} has dimensions
$\mat{P},\mat{Q}\in\R^{A\times 1}$, then \nashgeneral~reduces to
finding $\phi\in\Delta_A$ and $\psi\equiv 1$ such that
\begin{equation*}
  \phi^\top\mat{P} = \max_{\tilde{\phi}} \tilde{\phi}^\top\mat{P}
\end{equation*}
(the other side is trivialized as $\psi\in\Delta_1$ has only one
choice), which is solved at $\phi=e_{a^\star}$ where
$a^\star=\argmax_{a\in[A]} \mat{P}_a$. The situation is similar if
$\mat{P},\mat{Q}\in\R^{1\times B}$.

We summarize the above results into the following corollary.
\begin{corollary}[Regret bound for VI-ULCB on turn-based games]
  \label{cor:main_turn}
  For turn-based zero-sum Markov
  games, Algorithm \ref{algorithm:ucbvi_sim} has runtime
  $\poly(S,A,B,T)$ and achieves regret bound
  $\cO(\sqrt{H^3S^2(A+B)T\iota})$ with probability at least $1-p$,
  where $\iota = \log(SABT/p)$.
\end{corollary}

\section{Computationally Efficient Algorithm}
\label{section:efficient-algorithm}
In this section, we show that the computational issue of
Algorithm~\ref{algorithm:ucbvi_sim} is not intrinsic to the problem:
there exists a sublinear regret algorithm for general zero-sum Markov games
that has a guaranteed polynomial runtime, with regret scaling
as $\cO(T^{2/3})$, slightly worse than that of
Algorithm~\ref{algorithm:ucbvi_sim}. Therefore, computational
efficiency can be achieved if one is willing to trade some statistical
efficiency (sample complexity). For simplicity, we assume in this section that the initial state $s_1$ is fixed.

\begin{algorithm}[t]
   \caption{Value Iteration after Exploration (VI-Explore)}
   \label{algorithm:VI_explore}
\begin{algorithmic}[1]
   \STATE $(\hat{\P}, \hat{r}) \leftarrow \textsc{Reward\_Free\_Exploration}(\epsilon)$.
   \STATE $V_H(s) \leftarrow 0$ for any $s \in \cS_H$.
   \FOR{step $h=H-1,\dots,1$}
   \FOR{$(s, a, b)\in\cS\times\cA\times \cB$}
   \STATE $Q_{h}(s, a, b)\setto \hat{r}_h(s, a, b) +
   [\hat{\P}_{h} V_{h+1}](s, a, b)$. 
   \ENDFOR
   \FOR{$s \in \cS$}
   \STATE $(\hat{\mu}_h(\cdot|s), \hat{\nu}_h(\cdot|s))\leftarrow$\\ 
   \qquad $\textsc{Nash\_Zero\_Sum}(Q_h(s, \cdot, \cdot))$.
   \ENDFOR
   \ENDFOR
   \FOR{all remaining episodes}
   \STATE Play the game with policy $(\hat{\mu}, \hat{\nu})$.
   \ENDFOR
\end{algorithmic}
\end{algorithm}

\paragraph{Value Iteration after Exploration}
At a high level, our algorithm follows an explore-then-exploit
approach. We begin by running a (polynomial time) reward-free exploration
procedure \explore$(\epsilon)$ on a small number of episodes, which
queries the MDP and outputs an estimate $(\hat{\P}, \hat{r})$. Then,
we run value iteration on the empirical version of Markov game with
transition $\hat{\P}$ and reward $\hat{r}$, which finds its Nash
equilibrium $(\hat{\mu}, \hat{\nu})$. Finally, the algorithm simply
plays the policy $(\hat{\mu}, \hat{\nu})$ for the remaining
episodes. The full algorithm is described in Algorithm
\ref{algorithm:VI_explore} in the Appendix.

By ``reward-free'' exploration, we mean the procedure will not use any
reward information to guide exploration. Instead, the procedure
prioritize on visiting all possible states and gathering sufficient
information about their transition and rewards, so that $(\hat{\P},
\hat{r})$ are close to $(\P, r)$ in the sense that the Nash equilibria
of $\mdp(\hat{\P}, \hat{r})$ and $\mdp(\P, r)$ are close, where
$\mdp(\hat{\P}, \hat{r})$ denotes the Markov game with transition
$\hat{\P}$ and reward $\hat{r}$. 


This goal can be achieved by the following algorithm. For any fixed
state $s$, we can create an artificial reward $\tilde{r}$ defined as
$\tilde{r}(s, a, b) = 1$ and $\tilde{r}(s', a, b) = 0$ for any $s'\neq
s$, $a$ and $b$. Then, we can treat $\mathcal{C} = \cA\times \cB$ as a
new action set for a single agent, and run any standard reinforcement
learning algorithm with PAC or regret guarantees to find a
near-optimal policy $\tilde{\pi}$ of MDP$(H, \cS, \mathcal{C}, \P,
\tilde{r})$. It can be shown that the optimal policy for this MDP is
the policy that maximize the probability to reach state
$s$. Therefore, by repeatedly playing $\tilde{\pi}$, we can gather
transition and reward information at state $s$ as well as we
can. Finally, we repeat the routine above for all state $s$. See
Appendix \ref{appendix:proof-polytime_guarantee} for more details.

In this paper, we adapt the sharp treatments in \citet{jin2019reward} which studies reward-free exploration in the single-agent MDP setting, and provide following guarantees for the \explore~procedure.

\begin{theorem}[PAC bound for VI-Explore] \label{thm:PAC_explore}
With probability at least $1-p$, \explore$(\epsilon)$  runs for
$c(H^5S^2AB\iota/\epsilon ^2 + H^7S^{4}AB \iota ^3/\epsilon)$ episodes
with some large constant $c$, and $\iota = \log(HSAB/(p\epsilon))$,
and outputs $(\hat{\P}, \hat{r})$ such that the Nash equilibrium
$(\hat{\mu}, \hat{\nu})$ of \mdp$(\hat{\P}, \hat{r})$ satisfies
\begin{equation*}
\brac{V^{\dagger, \what{\nu}}(s_1) -
    V^{\what{\mu}, \dagger}(s_1)} \le \epsilon.
\end{equation*}
\end{theorem}



Importantly, such Nash equilibrium $(\hat{\mu}, \hat{\nu})$ of \mdp$(\hat{\P}, \hat{r})$ can be computed by Value Iteration (VI) using $\hat{\P}$ and $\hat{r}$.
VI only calls \nashzero~subroutine, which takes a matrix $\mat{Q} \in \R^{A\times B}$ and returns the Nash
equilibrium $(\phi, \psi) \in \Delta_A \times \Delta_B$ for zero-sum game, which satisfies
\begin{align}
  \label{equation:nash-zero-sum}
  \max_{\tilde{\phi}} \tilde{\phi}\trans
  \mat{Q} \psi = \phi\trans \mat{Q} \psi =
  \min_{\tilde{\psi}} 
  \phi\trans \mat{Q} \tilde{\psi}.
\end{align}
This problem can by solved efficiently (in polynomial time) by many existing algorithms designed for convex-concave optimization (see, e.g.~\citep{koller1994fast}), and does not suffer from the
PPAD-completeness that \nashgeneral~does.

The PAC bound in Theorem~\ref{thm:PAC_explore} can be easily converted
into a regret bound, 
which is presented as follows.




\begin{corollary}[Polynomial time algorithm via explore-then-exploit]
  \label{cor:polytime_guarantee}
  For zero-sum Markov games, with probability at least $1-p$,
  Algorithm~\ref{algorithm:VI_explore} runs in
  $\poly(S, A, B, H, T)$ time, and achieves regret bound
  \begin{equation*}
  \cO\paren{ (H^5S^2AB T^2 \iota)^{\frac{1}{3}} + \sqrt{H^7S^4AB T\iota^3}},
  \end{equation*}
  where $\iota = \log(SABT/p)$.
\end{corollary}


\section{Towards the Optimal Regret}
\label{section:optimality}
We investigate the tightness of our regret upper bounds in
Theorem~\ref{thm:main_sim}  and Corollary~\ref{cor:main_turn} through
raising the question of optimal regret in two-player Markov games,
and making initial progresses on it by providing lower bounds and
new upper bounds in specific settings. Specifically, we ask an

{\bf Open question:} What is the optimal regret for general Markov
games (in terms of dependence on $(H,S,A,B)$)?



It is known that the (tight) regret lower bound for single-player MDPs
is $\Omega(\sqrt{SAT \cdot \poly(H)})$~\citep{azar2017minimax}. By
restricting two-player games to a single-player MDP (making
the other player dummy), we immediately have

\begin{corollary}[Regret lower bound, corollary
  of~\citet{jaksch2010near}, Theorem 5] \label{cor:lowerbound}
  The regret\footnote{This also applies to the weak regret defined
    in~\eqref{equation:weak-regret}.
  } for any algorithm on
  turn-based games (and thus 
  also general zero-sum games) is lower bounded by
  $\Omega(\sqrt{H^2S(A+B)T})$.
\end{corollary}
Comparing this lower bound with the upper bound in
Theorem~\ref{thm:main_sim} ($\tlO(\sqrt{S^2ABT\cdot \poly(H)})$
regret for general games and $\tlO(\sqrt{S^2(A+B)T\cdot \poly(H)})$
regret for turn-based games), there are gaps in both the
$S$-dependence and the $(A,B)$-dependence.

\paragraph{Matching the lower bound on short-horizon games} 
Towards closing the gap between lower and upper bounds, we develop
alternative algorithms in the special case where \emph{each player
  only plays once}, i.e. one-step general games with 
$H=1$ and two-step turn-based games. In these cases, we show that
there exists mirror descent type algorithms that achieve an improved
regret $\tlO(\sqrt{S(A+B)T})$ (and thus matching the lower bounds),
\emph{provided that we consider a weaker notion of the regret} defined
as
\begin{definition}[Weak Regret]
  The \emph{weak regret} for any algorithm that deploys policies
  $(\mu^k,\nu^k)$ in episode $k$ is defined as
  \begin{equation}
    \label{equation:weak-regret}
    \begin{aligned}
       \WeakReg(K) \defeq \max_{\mu} \sum_{k=1}^K V^{\mu,
        \nu^k}(s_1^k) - \min_{\nu} \sum_{k=1}^K V^{\mu^k,
        \nu}(s_1^k).
    \end{aligned}
  \end{equation}
\end{definition}
The difference in the weak regret is that it uses \emph{fixed}
opponents---as opposed to adaptive opponents---for measuring the
performance gap: the max is taken with
respect to a fixed $\mu$ for all episodes $k=1,\dots,K$, rather than
a different $\mu$ for each episode. By definition, we have for any
algorithm that $\WeakReg(K) \le \Reg(K)$.


With the definition of the weak regret in hand, we now present our
results for one-step games. Their proofs can be found in
Appendix~\ref{appendix:proof-optimality}. 

\begin{theorem}[Weak regret for one-step simultaneous game, adapted
  from~\citet{rakhlin2013optimization}]
  \label{theorem:simultaneous-mirror-descent}
  For one-step simultaneous games ($H=1$),
  there exists a mirror descent type algorithm that achieves weak
  regret bound $\WeakReg(T)\le \tlO(\sqrt{S(A+B)T})$ with high
  probability.
\end{theorem}

\begin{theorem}[Weak regret for two-step turn-based game]
  \label{theorem:turn-based-mirror-descent}
  For one-step turn-based games ($H=2$), there exists a mirror descent
  type algorithm that 
  achieves weak regret bound
  $\WeakReg(T)\le \tlO(\sqrt{S(A+B)T})$ with high probability.
\end{theorem}
\paragraph{Proof insights; bottleneck in multi-step
  case}
The improved regret bounds in
Theorem~\ref{theorem:simultaneous-mirror-descent}
and~\ref{theorem:turn-based-mirror-descent} are possible due to
availability of \emph{unbiased estimates of counterfactual Q
  values}, which in turn can be used in
mirror descent type algorithms with guarantees. Such unbiased
estimates are only achievable in one-step games as the two policies
are ``not intertwined'' in a certain sense.
In contrast, in multi-step games (where each player plays more than
once), such unbiased estimates of counterfactual Q values are no
longer available, and it is unclear how to construct a mirror descent
algorithm there. We believe it would be an important open question to
close the gap in multi-step games (as well as the gap between regret
and weak regret) for a further understanding of exploration in games.








\section{Conclusion}
In this paper, we studied the sample complexity of finding the
equilibrium policy in the setting of competitive reinforcement
learning, i.e. zero-sum Markov games with two players. We designed a
self-play algorithm for zero-sum games and showed that it can
efficiently find the Nash equilibrium policy in the exploration
setting through establishing a regret bound. Our algorithm---Value
Iteration with Upper and Lower Confidence Bounds---builds on a
principled extension of UCB/optimism into the two-player case by
constructing upper and lower bounds on the value functions and
iteratively solving general sum subgames.

Towards investigating the optimal runtime and sample complexity in
two-player games, we provided accompanying results showing that (1)
the computational efficiency of our algorithm can be improved by
explore-then-exploit type algorithms, which has a slightly worse regret;
(2) the state and action space dependence in the regret can be reduced
in the special case of one-step games via alternative mirror descent
type algorithms.

We believe this paper opens up many interesting directions for future
work. For example, can we design a computationally efficient
algorithms that achieves $\tlO(\sqrt{T})$ regret? What are the optimal
dependence of the regret on $(S,A,B)$ in multi-step games? Also, the
present results only work in tabular games, and it would be of
interest to investigate if similar results can hold in presence of
function approximation.
\section*{Acknowledgments}
We thank Sham Kakade and Haipeng Luo for valuable discussions on the
related work. We also thank the Simons Institute at Berkeley and its
Foundations of Deep Learning program in Summer 2019 for
hosting the authors and incubating our initial discussions.

\bibliographystyle{plainnat}
\bibliography{bib}

\appendix
\section{Proofs for Section~\ref{section:main}}
\subsection{Proof of Theorem~\ref{thm:main_sim}}
\label{appendix:proof-main_sim}

\paragraph{Notation:} To be clear from the context, we denote the
upper bound and lower bound $Q^\up$ and $Q^\low$ computed at the
$k$-th episode as $Q^{\up, k}$ and $Q^{\low, k}$, and policies
computed and used at the $k$-th episode as $\mu^k$ and $\nu^k$. 

\paragraph{Choice of bonus:} $\beta_t = c\sqrt{SH^2\iota/t}$ for sufficient large absolute constant $c$.

\begin{lemma}[ULCB]\label{lem:ULCB_sim}
With probability at least $1-p$, we have following bounds for any $(s, a, b, h, k)$:
\begin{align}
&V_h^{\up, k}(s) \ge \sup_{\mu} V_h^{\mu, \nu^k}(s),  &Q_h^{\up, k}(s, a, b) \ge& \sup_{\mu} Q_h^{\mu, \nu^k}(s, a, b) \label{eq:ucb_sim}\\
&V_h^{\low, k}(s) \le \inf_{\nu} V_h^{\mu^k, \nu}(s), &Q_h^{\low, k}(s, a, b) \le& \inf_{\nu} Q_h^{\mu^k, \nu}(s, a, b) \label{eq:lcb_sim}
\end{align}
\end{lemma}
\begin{proof}
By symmetry, we only need to prove the statement \eqref{eq:ucb_sim}. For each fixed $k$, we prove this by induction from $h=H+1$ to $h=1$. For base case, we know at the $(H+1)$-th step, 
$V_{H+1}^{\up, k}(s) = \sup_{\mu} V_{H+1}^{\mu, \nu^k}(s) = 0$.

Now, assume the left inequality in \eqref{eq:ucb_sim} holds for $(h+1)$-th step, for the $h$-th step, we first recall the updates for $Q$ functions respectively:
\begin{align*}
Q_{h}^{\up, k}(s, a, b) = & \min\set{r_h(s, a, b) +
   [\widehat{\P}_{h}^k V_{h+1}^{\up, k}](s, a, b) +
   \beta_t, H}\\
\sup_{\mu} Q_h^{\mu, \nu^k}(s, a, b)
=& r_h(s, a, b) + [\P_{h} \sup_{\mu}V_{h+1}^{\mu, \nu^k}](s, a, b)
\end{align*}
In case of $Q_{h}^{\up, k}(s, a, b) = H$, the right inequality in \eqref{eq:ucb_sim} clearly holds. Otherwise, we have:
\begin{align*}
Q_{h}^{\up, k}(s, a, b) - \sup_{\mu} Q_h^{\mu, \nu^k}(s, a, b)
=& [\widehat{\P}_{h}^k V_{h+1}^{\up, k}](s, a, b) - [\P_{h}^k \sup_{\mu}V_{h+1}^{\mu, \nu^k}](s, a, b) +  \beta_t\\
=& [\widehat{\P}_{h}^k (V_{h+1}^{\up, k} - \sup_{\mu}V_{h+1}^{\mu, \nu^k})](s, a, b) - [(\widehat{\P}_{h}^k -\P_{h}) \sup_{\mu}V_{h+1}^{\mu, \nu^k}](s, a, b) +  \beta_t
\end{align*}
Since $\widehat{\P}_{h}^k$ preserves the positivity, by induction assumption, we know the first term is positive. By Lemma \ref{lem:uniform_concern_sim}, we know the second term $\ge -\beta_t$. This finishes the proof of the right inequality in \eqref{eq:ucb_sim}.

To prove the left inequality in \eqref{eq:ucb_sim}, again recall the updates for $V$ functions respectively:
\begin{align*}
  V_{h}^{\up, k}(s) =
  & \mu^k_h(s)\trans Q_h^{\up, k}(s, \cdot, \cdot)
    \nu^k_h(s) = \max_{\phi \in \Delta_\cA}
    \phi\trans Q_h^{\up, k}(s, \cdot, \cdot)
    \nu^k_h(s) \\ 
  \sup_{\mu} V_h^{\mu, \nu^k}(s) =
  &\max_{\phi \in \Delta_\cA}
    \phi\trans [\sup_{\mu} Q_h^{\mu,
    \nu^k}(s, \cdot, \cdot)] \nu^k_h(s) 
\end{align*}
where the first equation is by the definition of policy $\mu^k$ the algorithm picks. Therefore:
\begin{equation*}
V_{h}^{\up, k}(s) - \sup_{\mu} V_h^{\mu, \nu^k}(s)
\ge \max_{\phi \in \Delta_\cA} \phi\trans [Q_h^{\up, k} - \sup_{\mu} Q_h^{\mu, \nu^k}](s, \cdot, \cdot) \nu^k_h(s) \ge 0.
\end{equation*}
This finishes the proof.
\end{proof}


\begin{lemma}[Uniform Concentration]\label{lem:uniform_concern_sim}
  Consider value function class
  \begin{equation*}
    \mathcal{V}_{h+1} =
    \set{V:\cS_{h+1}\to\R~\mid~V(s)\in[0,H]~\textrm{for
        all}~s\in \cS_{h+1}}.
  \end{equation*}
  There exists an absolute constant $c$, with probability at least $1-p$, we have:
  \begin{equation*}
    \abs{[(\hat{\P}^k_h - \P_h)V](s, a, b)} \le
    c\sqrt{SH^2\iota/N^k_h(s, a, b)} \quad \textrm{for all}~ (s, a, b,
    k, h) ~\textrm{and all}~ V\in \mathcal{V}_{h+1}. 
  \end{equation*}
\end{lemma}
\begin{proof}
  We show this for one $(s,a,b,k,h)$; the rest follows from a union
  bound over these indices (and results in a larger logarithmic
  factor.) Throughout this proof we let $c>0$ to be an absolute
  constant that may vary from line to line.

  Let $\mc{V}_\eps$ be an $\eps$-covering of $\mc{V}_{h+1}$ in
  the $\infty$ norm (that is, for any $V\in\mc{V}_{h+1}$ there exists
  $\what{V}\in\mc{V}_\eps$ such that $\sup_s |V(s) - \what{V}(s)| \le 
  \eps$.) We have $|\mc{V}_\eps|\le (1/\eps)^S$, and by Hoeffding
  inequality and a union bound (over both $\what{V}$ and
  $N_h^k\in[K]$), we have with probability at least
  $1-p$ that
  \begin{align*}
    \abs{\sup_{\what{V}\in\mc{V}_\eps} \brac{(\hat{\P}_h^k - \P_h)
    \what{V}}} \le
    \sqrt{\frac{H^2(S\log(1/\eps) + \log(K/p))}{N_h^k(s,a,b)}}.
  \end{align*}
  Taking $\eps=c\sqrt{H^2S\iota/K}$, the above implies that
  \begin{align*}
    \abs{\sup_{\what{V}\in\mc{V}_\eps} \brac{(\hat{\P}_h^k -
    \P_h)\what{V}}} \le
    c\sqrt{\frac{H^2S\iota}{N_h^k(s, a, b)}}.
  \end{align*}
  Meanwhile, with this choice of $\eps$, for any $V\in\mc{V}_{h+1}$,
  there exists $\what{V}\in\mc{V}_\eps$ such that $\sup_s |V(s) -
  \what{V}(s)|\le \eps$, and therefore
  \begin{align*}
    \abs{\brac{(\hat{\P}_h^k - \P_h)V} - \brac{(\hat{\P}_h^k -
    \P_h)\what{V}}} \le 2\eps = c\sqrt{\frac{H^2S\iota}{K}} \le
    c\sqrt{\frac{H^2S\iota}{N_h^k(s,a,b)}}.
  \end{align*}
  Combining the preceding two bounds, we have that the desired
  concentration holds for all $V\in\mc{V}_{h+1}$.
  
\end{proof}

\begin{proof}[Proof of Theorem \ref{thm:main_sim}]
By Lemma \ref{lem:ULCB_sim}, we know the regret,
\begin{equation*}
\text{Regret}(K) = \sum_{k=1}^K \left[\sup_{\mu}\sind{V}{\mu, \nu^k}{1} (\ind{s}{k}{1}) - \inf_{\nu}\sind{V}{\mu^k, \nu}{1} (\ind{s}{k}{1})\right] 
\le \sum_{k=1}^K [V_1^{\up, k}(\ind{s}{k}{1})  - V_1^{\low, k}(\ind{s}{k}{1})]
\end{equation*}
On the other hand, by the updates in Algorithm \ref{algorithm:ucbvi_sim}, we have:
\begin{align*}
[V_h^{\up, k} - V_h^{\low, k}](\ind{s}{k}{h})
=& \mu^k_h(s^k_h)\trans [Q_h^{\up, k} - Q_h^{\low, k}](s^k_h, \cdot, \cdot) \nu^k_h(s^k_h), \\
=& [Q_h^{\up, k} - Q_h^{\low, k}](s^k_h, a^k_h, b^k_h) + \xi^k_h\\
\le & [\widehat{\P}_{h}^k (V_{h+1}^{\up, k} - V_{h+1}^{\low, k})](s^k_h, a^k_h, b^k_h) + 2\beta_h^k + \xi_h^k \\
\le & [\P (V_{h+1}^{\up, k} - V_{h+1}^{\low, k})](s^k_h, a^k_h, b^k_h) + 4\beta_h^k + \xi_h^k\\
= & (V_{h+1}^{\up, k} - V_{h+1}^{\low, k})(\ind{s}{k}{h+1}) + 4\beta_h^k + \xi_h^k + \zeta_h^k
\end{align*}
the last inequality is due to Lemma
\ref{lem:uniform_concern_sim}. (Recall that $\beta_h^k\defeq
\beta_{N_h^k(s_h^k, a_h^k,
  b_h^k)}=c\sqrt{H^2S\iota/N_h^k(s_h^k,a_h^k,b_h^k)}$ when $N_h^k\ge
1$. In the case when
  $N_h^k=0$, we can still define $\beta_h^k= \beta_0 \defeq
  c\sqrt{H^2S\iota}$, and the above inequality still holds as we have
  $Q_h^{\up, k}-Q_h^{\low,k}=H\le \beta_0$.)
Above, $\xi_h^k$ and $\zeta_h^k$ are defined as
\begin{align*}
\xi_h^k =& \E_{a \sim \mu^k_h(s^k_h), b\sim \nu^k_h(s^k_h)} [Q_h^{\up,
           k} - Q_h^{\low, k}](s^k_h, a, b)  - [Q_h^{\up, k} -
           Q_h^{\low, k}](s^k_h, a^k_h, b^k_h)\\ 
\zeta_h^k =& \E_{s \sim \P_h(\cdot|s^k_h, a^k_h, b^k_h)}
             [(V_{h+1}^{\up, k} - V_{h+1}^{\low, k})](s) -
             [V_{h+1}^{\up, k} - V_{h+1}^{\low, k})](\ind{s}{k}{h+1}) 
\end{align*}
Both $\xi_h^k$ and $\zeta_h^k$ are martingale difference sequence,
therefore by the Azuma-Hoeffding inequality we have with probability
$1-p$ that
\begin{equation*}
  \sum_{k, h} \xi_h^k \le \cO(\sqrt{HT\iota}) \quad {\rm and} \quad
  \sum_{k, h} \zeta_h^k \le \cO(\sqrt{HT\iota}). 
\end{equation*}
Therefore, by our choice of bonus $\beta_t$ and the Pigeonhole
principle, we have
\begin{align*}
  & \quad \sum_{k=1}^K \brac{V_1^{\up, k}(\ind{s}{k}{1})  - V_1^{\low,
    k}(\ind{s}{k}{1})}
    \le \sum_{k, h} \paren{4\beta^k_h + \xi^k_h + \zeta^k_h} \\
  & \le \sum_{h,s\in\cS_h,a\in\cA_h,b\in\cB_h} c\cdot
    \sum_{t=1}^{N^K_h(s,a,b)} \sqrt{\frac{H^2S\iota}{t}} + \cO(\sqrt{HT\iota}) \\
  & = \sum_{h,s\in\cS_h,a\in\cA_h,b\in\cB_h}  \cO\paren{\sqrt{H^2S\iota
    \cdot N^K_h(s, a, b)}} + \cO(\sqrt{HT\iota}) \\
  & \le \sum_{h\in[H]} \cO\paren{\sqrt{H^2S^2A_hB_hK\iota}} \le
    \cO\paren{\sqrt{H^4S^2\brac{\max_h A_hB_h}K\iota}} =
    \cO\paren{\sqrt{H^3S^2\brac{\max_h  A_hB_h}T\iota}}.
\end{align*}
This finishes the proof.
\end{proof}

\subsection{Proof of Corollary~\ref{corollary:pac}}
\label{appendix:proof-pac}
The proof is based on a standard online-to-batch conversion
(e.g. \citep[Section 3.1,][]{jin2018q}.)
Let $(\what{\mu}^k, \what{\nu}^k)$ denote the policies deployed by the
VI-ULCB algorithm in episode $k$. We sample $\what{\mu},\what{\nu}$
uniformly as
\begin{equation*}
  \what{\mu} \sim \Unif\set{\mu^1,\dots,\mu^K}~~~{\rm
    and}~~~\what{\nu} \sim \Unif\set{\nu^1,\dots,\nu^K}.
\end{equation*}
Taking expectation with respect to this sampling gives
\begin{align*}
  & \quad \E_{\what{\mu},\what{\nu}}\brac{ V^{\dagger,\what{\nu}}(s_1) -
  V^{\what{\mu},\dagger}(s_1) }  = \frac{1}{K}\sum_{k=1}^K \brac{ V^{\dagger,\nu^k}(s_1) -
  V^{\mu^k,\dagger}(s_1) } \\
  & = \frac{1}{K}\Reg(K) \le \tlO\paren{
  \frac{\sqrt{H^3S^2ABT}}{K} } \le \tlO\paren{
  \sqrt{\frac{H^4S^2AB}{K}} },
\end{align*}
where we have applied Theorem~\ref{thm:main_sim} to bound the regret
with high probability.
Choosing $K\ge \tlO(H^4S^2AB/\epsilon^2)$, the right hand side is upper
bounded by $\epsilon$, which finishes the proof.

\section{Proofs for Section \ref{section:efficient-algorithm}}\label{appendix:proof-polytime_guarantee}
In this section, we prove Theorem \ref{thm:PAC_explore} and Corollary \ref{cor:polytime_guarantee} based on the following lemma about subroutine \explore. We will defer the proof of this Lemma to Appendix \ref{sec:proof_explore}.

\begin{lemma}\label{lem:policy_evaluation}
Under the preconditions of Theorem \ref{thm:PAC_explore}, with probability at least $1-p$, for any policy $\mu, \nu$, we have:
\begin{equation} 
|\hat{V}^{\mu, \nu}_{1}(s_1) -V^{\mu, \nu}_{1}(s_1)| \le \epsilon/2
\end{equation}
where $\hat{V}, V$ are the value functions of $\mdp(\hat{\P}, \hat{r})$ and $\mdp(\P, r)$.
\end{lemma}

\subsection{Proof of Theorem \ref{thm:PAC_explore}}
Since both $\inf$ and $\sup$ are contractive maps, by Lemma \ref{lem:policy_evaluation}, we have:
\begin{align*}
&|\inf_\nu V_1^{\hat{\mu}, \nu}(s_1) - \inf_\nu \hat{V}_1^{\hat{\mu}, \nu}(s_1)| \le \epsilon/2 \\
&|\sup_\mu V_1^{\mu, \hat{\nu}}(s_1) - \sup_\mu \hat{V}_1^{\mu, \hat{\nu}}(s_1)| \le \epsilon/2
\end{align*}
Since $(\hat{\mu},\hat{\nu})$ are the Nash Equilibria for $\mdp(\hat{\P},
\hat{r})$, we have $\inf_\nu \hat{V}_1^{\hat{\mu}, \nu}(s_1) =
\sup_\mu \hat{V}_1^{\mu, \hat{\nu}}(s_1)$. This gives: 
\begin{align*}
\sup_\mu V_1^{\mu, \hat{\nu}}(s_1) - \inf_\nu V_1^{\hat{\mu}, \nu}(s_1)
\le& |\sup_\mu V_1^{\mu, \hat{\nu}}(s_1) - \sup_\mu \hat{V}_1^{\mu, \hat{\nu}}(s_1)| 
+ | \sup_\mu \hat{V}_1^{\mu, \hat{\nu}}(s_1) - \inf_\nu \hat{V}_1^{\hat{\mu}, \nu}(s_1)|\\
&+ | \inf_\nu \hat{V}_1^{\hat{\mu}, \nu}(s_1) - \inf_\nu V_1^{\hat{\mu}, \nu}(s_1)| \le \epsilon.
\end{align*}
which finishes the proof.

\subsection{Proof of Corollary \ref{cor:polytime_guarantee}}
Recall that Theorem \ref{thm:PAC_explore} requires
$T_0 = c(H^5S^2AB\iota/\epsilon ^2 + H^7S^{4}AB \iota ^3/\epsilon)$ episodes to obtain an $\epsilon$-optimal policies in the sense:
\begin{equation*}
\sup_\mu V_1^{\mu, \hat{\nu}}(s_1) - \inf_\nu V_1^{\hat{\mu}, \nu}(s_1) \le \epsilon.
\end{equation*}

Therefore, if the agent plays the Markov game for $T$ episodes, it can use first $T_0$ episodes to explore to find $\epsilon$-optimal policies $(\hat{\mu}, \hat{\nu})$, and use the remaining $T-T_0$ episodes to exploit (always play $(\hat{\mu}, \hat{\nu})$). Then, the total regret will be upper bounded by:
\begin{equation*}
\Reg(K) \le T_0 \times 1 + (T - T_0) \times \epsilon
\end{equation*}
Finally, choose
\begin{equation*}
\epsilon = \max \left\{\left(\frac{H^5S^2AB\iota}{T}\right)^{\frac{1}{3}}, \left(\frac{H^7S^4AB\iota^3}{T}\right)^{\frac{1}{2}}\right\}
\end{equation*}
we finishes the proof.
\section{Proofs for Section~\ref{section:optimality}}
\label{appendix:proof-optimality}

\subsection{Proof of
  Theorem~\ref{theorem:simultaneous-mirror-descent}}
The theorem is almost an immediate consequence of the general result
on mirror descent~\citep{rakhlin2013optimization}. However, for
completeness, we provide a self-contained proof here. The main
ingredient in our proof is to show that a ``natural'' loss estimator
satisfies desirable properties---such as unbiasedness and bounded
variance---for the standard analysis of mirror descent type algorithms
to go through.

\paragraph{Special case of $S=1$}
We first deal with the case of $S=1$. As the game only has one step
($H=1$), it reduces to
a zero-sum matrix game with a noisy bandit feedback, i.e.
there is an unknown payoff matrix $\mat{R}\in[0,1]^{A\times B}$,
the algorithm plays policies $(\mu_k, \nu_k)\in\Delta_A\times
\Delta_B$, observes feedback $r(a^k, b^k)=\mat{R}_{a^k,b^k}$ where
$(a^k, b^k)\sim\mu_k\times \nu_k$, and the weak regret has form
\begin{equation*}
  \WeakReg(T) = \max_{\mu} \sum_{k=1}^T \mu^\top \mat{R}\nu_k -
  \min_{\nu} \sum_{k=1}^K \mu_k^\top \mat{R}\nu.
\end{equation*}
Note that this regret can be decomposed as
\begin{equation*}
  \WeakReg(T) = \underbrace{\max_{\mu} \sum_{k=1}^T \mu^\top
    \mat{R}\nu_k - \sum_{k=1}^T 
  \mu_k^\top\mat{R}\nu_k}_{\rm I} + \underbrace{\sum_{k=1}^T
  \mu_k^\top\mat{R}\nu_k - 
  \min_{\nu} \sum_{k=1}^T \mu_k^\top \mat{R}\nu}_{\rm II}.
\end{equation*}
We now describe the mirror descent algorithm for the max-player and
show that it achieves bound ${\rm I}\le \tlO(\sqrt{AT})$ regardless of
the strategy of the min-player. A similar argument on the min-player
will yield a regret bound ${\rm II}\le \tlO(\sqrt{BT})$ on the second
part of the above regret and thus show
$\WeakReg(T)\le \tlO(\sqrt{(A+B)T})$.

For all $k\in[T]$, define the loss vector $\ell_k\in\R^A$ for the
max-player as
\begin{equation*}
  \ell_k(a) \defeq e_a^\top \mat{R} \nu_k,~~~\textrm{for
    all}~a\in\mc{A}. 
\end{equation*}
With this definition the regret I can be written as
\begin{equation*}
  {\rm I} = \max_{a} \sum_{k=1}^T \ell_k(a) - \sum_{k=1}^T
  \mu_k(a)\ell_k(a). 
\end{equation*}
Now, define the loss estimate $\wt{\ell}_k(a)$ as
\begin{equation*}
  \wt{\ell}_k(a) \defeq 1 - \frac{\indic{a^k=a}}{\mu_k(a)} \brac{1 -
    r(a, b^k)}. 
\end{equation*}
We now show that this loss estimate satisfies the following properties:
\begin{enumerate}[(1)]
\item Computable: the reward
  $r(a,b^k)$ is seen when $a=a^k$, and the loss estimate is equal to 1
  for all $a\neq a^k$.
\item Bounded: we have $\wt{\ell}_k(a) \le 1$ almost surely for all
  $k$ and $a$.
\item Unbiased estimate of $\ell_k(a)$. For any fixed state
  $a\in\mc{A}$, we have
  \begin{align*}
    & \quad \E\brac{\wt{\ell}_k(a) | \mc{F}_{k-1}} = 1 - \mu_k(a)
      \cdot \frac{1}{\mu_k(a)} \E_{b^k\sim \nu_k}\brac{1 - r(a, b^k)} \\
    & = 1 - \paren{1 - \E_{b^k\sim\nu_k}[r(a, b^k)]} =
      E_{b^k\sim\nu_k}[r(a, b^k)] = e_a^\top \mat{R}\nu_k = \ell_k(a).
  \end{align*}
\item Bounded variance: one can check that
  \begin{align*}
    & \quad \E\brac{ \sum_{a\in\mc{A}} \mu_k(a)\wt{\ell}_k(a)^2
      \big| \mc{F}_{k-1}} \\
    & = \E_{b^k\sim\nu_k}\brac{\sum_{a\in\mc{A}} \mu_k(a) \paren{1 - 2\paren{1
      - r(a, b^k)}} + 
      \sum_{a\in\mc{A}} (1 - r(a, b^k))^2}.
  \end{align*}
  Letting $y_a\defeq 1 - r(a, b^k)$, we have $y_a\in[0,1]$ almost
  surely (though it is random), and thus
  \begin{align*}
    \sum_a \mu_k(a)(1 - 2y_a) + \sum_a y_a^2  \le 1 - 2\min_a y_a +
    \sum_a y_a^2 = \sum_{a\neq a_*} y_a^2 + (y_{a^\star} - 1)^2 \le A,
  \end{align*}
  where $a^\star=\argmin_{a\in\mc{A}} y_a$.
\end{enumerate}
Therefore, adapting the proof of standard regret-based bounds for
the mirror descent (EXP3) algorithm (e.g.~\citep[Theorem
11.1]{lattimore2018bandit}), using the loss estimate $\wt{\ell_k}(a)$
and taking the step-size to be $\eta_+\equiv \sqrt{\log A/AT}$,
we have the regret bound
\begin{equation*}
  \WeakReg_+ \le C\cdot \sqrt{AT\log A},
\end{equation*}
where $C>0$ is an absolute constant. This shows the desired bound
$\tlO(\sqrt{AT})$ for term I in the regret, and a similar bound
$\tlO(\sqrt{BT})$ holds for term II by using the same algorithm on the
min-player. 

\paragraph{Case of $S>1$}
The case of $S>1$ can be viewed as $S$ independent zero-sum matrix
games. We can let both players play the each matrix game independently
using an adaptive step-size sequence (such as the EXP3++
algorithm of~\citet{seldin2014one}) so that on the game with initial
state $s\in\mc{S}$ they achieve regret bound
\begin{equation*}
  \tlO(\sqrt{(A+B)T_s}),
\end{equation*}
where $T_s$ denotes the number of games that has context $s$. Summing
the above over $s\in\mc{S}$ gives the regret bound
\begin{equation*}
  \WeakReg(T) \le \sum_s \tlO(\sqrt{(A+B)T_s}) \le \tlO(\sqrt{S(A+B)T}),
\end{equation*}
as $\sum_s T_s = T$ and thus $\sum_s \sqrt{T_s}\le \sqrt{ST}$ by
Cauchy-Schwarz. 

\qed

\subsection{Proof of Theorem~\ref{theorem:turn-based-mirror-descent}}
We first describe our algorithm for one-step turn-based games ($H=2$.)
Note that this is not equivalent to a zero-sum matrix game, as there
is an unknown transition dynamics involved.


As both the max and min player only have one turn to
play: $\mu=\set{\mu_1}$ and $\nu=\set{\nu_2}$, in this section we
will abuse notation slightly and use $(\mu, \nu)$ to denote
$(\mu_1,\nu_2)$. We will also use $(\mc{A}, \mc{B})$ to denote
$(\mc{A}_1, \mc{B}_2)$ for similar reasons.

We now present our mirror descent based algorithm for one-step
turn-based games. Define the loss estimates
\begin{align}
    & \wt{Q}_1^k(s_1^k, a) \defeq 2 - 
    \frac{\indic{a^k=a}}{\mu^k(a|s_1^k)} \cdot \brac{2 - (r(s_1^k, a) +
      r(s_2^k, b^k))}~~~\textrm{for
      all}~a\in\mc{A}, \label{equation:q1-estimate} \\ 
    & \wt{Q}_2^k(s_2^k, b) \defeq 1 - 
    \frac{\indic{b^k=b}}{\nu^k(b|s_2^k)} \cdot \brac{1 - r(s_2^k,
      b)}~~~\textrm{for
      all}~b\in\mc{B}. \label{equation:q2-estimate} 
\end{align}
The full algorithm is described in
Algorithm~\ref{algorithm:mirror-descent}.

\begin{algorithm*}[t]
   \caption{Mirror descent for one-step turn-based game}
   \label{algorithm:mirror-descent}
   \begin{algorithmic}
     \INPUT Learning rate schedule $(\eta_{+,k}(s), \eta_{-,k}(s))$.
     \STATE {\bf Initialize}: Set $(\mu,\nu)$ to be uniform:
     $\mu(a|s_1)=\frac{1}{A}$ for all $(s_1,a)$ and
     $\nu(b|s_2)=\frac{1}{B}$ for all $(s_2,b)$.
     \FOR{episode $k=1,\dots,K$}
     \STATE Receive $s_1$.
     \STATE Play action $a\sim \mu(\cdot|s_1)$. Observe reward
     $r_1(s_1, a)$ and next state $s_2$.
     \STATE Play action $b\sim \nu(\cdot|s_2)$. Observe reward
     $r_2(s_2, b)$.
     \STATE Compute $\set{\wt{Q}_1^k(s_1^k,a)}_{a\in\mc{A}}$ according
     to~\eqref{equation:q1-estimate} and update
     $$\mu^{k+1}(a|s_1^k)  \propto\mu^k(a|s_1^k) \cdot \exp(\eta_{+,k}(s_1^k)
     \wt{Q}_1^k(s_1^k,a)).$$
     \STATE Compute $\set{\wt{Q}_2^k(s_2^k, b)}_{b\in\mc{B}}$
     according to~\eqref{equation:q2-estimate} and update
     $$\nu^{k+1}(b|s_2^k) \propto 
     \nu^{k}(b|s_2^k) \cdot \exp(-\eta_{-,k}(s_2^k) \wt{Q}_2^k(s_2^k, b)).$$
     \ENDFOR
   \end{algorithmic}
\end{algorithm*}


We are now in position to prove the theorem.

\begin{proof-of-theorem}[\ref{theorem:turn-based-mirror-descent}]
We begin by decomposing the weak regret into two parts:
\begin{align*}
  & \quad \WeakReg(T) = \max_{\mu} \sum_{k=1}^K V_1^{\mu,
    \nu^k}(s_1^k) 
    - \min_{\nu} \sum_{k=1}^K V_1^{\mu^k, \nu}(s_1^k) \\
  & = \underbrace{\max_{\mu} \sum_{k=1}^K V_1^{\mu, \nu^k}(s_1^k) -
    \sum_{k=1}^K V_1^{\mu^k, \nu^k}(s_1^k)}_{\WeakReg_+} +
    \underbrace{\sum_{k=1}^K
    V_1^{\mu^k, \nu^k}(s_1^k) - \min_{\nu} \sum_{k=1}^K
    V_1^{\mu^k, \nu}(s_1^k)}_{\WeakReg_-}.
\end{align*}
In the following, we show that both $\WeakReg_+\le\cO(\sqrt{SAT\iota})$
and $\WeakReg_-\le \cO(\sqrt{SBT\iota})$, 
which when combined gives the desired result.

\paragraph{Bounding $\bf \WeakReg_+$} We first consider the case that the
initial state is fixed, i.e. $s_1^k\equiv s_1$ for
some fixed $s_1\in\mc{S}_1$ and all $k$.
In this case, we have for any $\mu$ that
\begin{align*}
  V_1^{\mu, \nu^k}(s_1) = \sum_{a\in\mc{A}} \mu(a|s_1) Q_1^{\mu,
  \nu^k}(s_1, a) = \<Q_1^{\mu, \nu^k}(s, \cdot),
  \mu(\cdot|s_1)\>_a = \<Q_1^{\cdot, \nu^k}(s_1, \cdot),
  \mu(\cdot|s_1)\>_a.
\end{align*}
Above, the last equality follows by the fact the max player will not
play again after the initial action in one-step games, i.e. 
$Q_1^{\mu,\nu}(s,a)$ does not depend on $\mu$. Applying the
above expression, $\WeakReg_+$ can be rewritten as
\begin{align*}
  \WeakReg_+ = \max_{\mu} \sum_{k=1}^K \<Q_1^{\cdot, \nu^k}(s_1,
  \cdot), \mu(\cdot|s_1)\>_a - \sum_{k=1}^K \<Q_1^{\cdot, \nu^k}(s_1,
  \cdot), \mu^k(\cdot|s_1)\>_a,
\end{align*}
Therefore, bounding $\WeakReg_+$ reduces to solving an online linear
optimization problem over $\Delta_{\mc{A}}$ with bandit
feedback, where at each step we play $\mu^k$ and then suffer a
linear loss with loss vector
$\set{Q_1^{\cdot,\nu^k}(s_1,\cdot)}_{a\in\mc{A}}$.

Now, recall that our loss estimate in~\eqref{equation:q1-estimate},
adapted to the setting that $s_1^k\equiv s_1$ can be written as:
\begin{align*}
  \wt{Q}_1^k(s_1, a) = 2 - \frac{\indic{a^k=a}}{\mu^k(a|s_1)}
  \cdot \brac{2 - (r(s_1, a) + r(s_2^k, b^k))}.
\end{align*}
We now show that this loss estimate satisfies the following properties:
\begin{enumerate}[(1)]
\item Computable: the reward
  $r(s_1,a)$ is seen when $a=a^k$, and the loss estimate is equal to
  2 for all other $a\neq a^k$.
\item Bounded: we have $\wt{Q}_1^k(s_1, a) \le 2$ for all $k$ and $a$.
\item Unbiased estimate of $Q_1^{\cdot, \nu^k}(s_1,\cdot)$. For any
  fixed state $a$, when $a^k=a$ happens, $s_2^k$ is drawn from the
  MDP transition $\P_1(\cdot|s_1, a)$. Therefore, letting
  $\mc{F}_{k-1}$ be the $\sigma$-algebra that encodes all the
  information observed at the end of episode $k-1$, we have that
  \begin{align*}
    \wt{Q}_1^k(s_1, a) | \mc{F}_{k-1} \eqnd 2 -
    \frac{\indic{a^k=a}}{\mu^k(a|s_1)} \cdot \brac{2 - r(s_1, a) -
    r(s_2^{(a)}, b^{(a)})},
  \end{align*}
  where $\eqnd$ denotes equal in distribution,
  $s_2^{(a)}\sim\P_1(\cdot|s_1,a)$ is an ``imaginary'' state had we played
  action $a$ at step 1, and
  $b^{(a)}\sim\nu^k(\cdot|s_2^{(a)})$. Therefore we have
  \begin{align*}
    & \quad \E\brac{\wt{Q}_1^k(s_1, a) \Big| \mc{F}_{k-1}} \\
    & = \E_{a\sim\mu^k(\cdot|s_1)}\brac{ 2 -
      \frac{\indic{a=a}}{\mu^k(a|s_1)} \E_{s_2^{(a)}, b^{(a)}}
      \brac{2 - r(s_1, a) - r(s_2^{(a)}, b^{(a)})}} \\
    & = \E_{s_2^{(a)},b^{(a)}}\brac{ 2 -
      \frac{\mu^k(a|s_1)}{\mu^k(a|s_1)} 
      \cdot \brac{2 - (r(s_1, a) + r(s_2^{(a)}, b^{(a)}))} } \\
    & = \E_{s_2^{(a)},b^{(a)}}
      [r(s_1, a) + r(s_2^{(a)}, b^{(a)})] = Q_1^{\cdot, \nu^k}(s_1,
    a).
  \end{align*}
\item Bounded variance: one can check that
  \begin{align*}
    & \quad \E\brac{ \sum_{a\in\mc{A}} \mu^k(a|s_1)\wt{Q}_1^k(s_1, a)^2
      \big| \mc{F}_{k-1}} \\
    & = 4\sum_{a\in\mc{A}} \mu^k(a|s_1) \paren{1 - \E_{s_2^{(a)},
      b^{(a)}}[2 - r(s_1, a) - r(s_2^{(a)}, b^{(a)})]} \\
    & \qquad +
      \sum_{a\in\mc{A}} \E_{s_2^{(a)}, b^{(a)}}[(2 - r(s_1, a) -
      r(s_2^{(a)}, b^{(a)}))^2]
  \end{align*}
  Letting $p_a\defeq \mu^K(a|s_1)$ and $y_a\defeq 2 - r(s_1, a) -
  r(s_2^{(a)}, b_2^ {(a)})$, we have $y_a\in[0,2]$ almost surely
  (though it is random), and thus
  \begin{align*}
    4\sum_a p_a(1 - y_a) + \sum_a y_a^2  \le 4(1 - \min_a y_a) +
    \sum_a y_a^2 = \sum_{a\neq a_*} y_a^2 + (y_{a*} - 2)^2 \le 4A,
  \end{align*}
  where $a*=\argmin_{a\in\mc{A}} y_a$.
\end{enumerate}

Therefore, adapting the proof of standard regret-based bounds for
the mirror descent (EXP3) algorithm (e.g.~\citep[Theorem
11.1]{lattimore2018bandit}), taking $\eta_+\equiv \sqrt{\log A/AT}$,
we have the regret bound
\begin{equation*}
  \WeakReg_+ \le C\cdot \sqrt{AT\log A},
\end{equation*}
where $C>0$ is an absolute constant.

In the general case where $s_1^k$ are not fixed and can be (in the
worst case) adversarial, the design of
Algorithm~\ref{algorithm:mirror-descent} guarantees that for any
$s\in\mc{S}$, $\mu(\cdot|s)$ gets updated after the $k$-th episode
only if $s_1^k=s$; otherwise the $\mu(\cdot|s)$ is left
unchanged. Therefore, the algorithm behaves like solving $S$
bandit problems independently, so we can sum up all the one-state
regret bounds of the above form and obtain that
\begin{equation*}
  \WeakReg_+ \le \sum_{s\in\mc{S}} C\sqrt{AT_s\log A} \stackrel{(i)}{\le}
  C\sqrt{SAT\log A} = \cO(\sqrt{SAT\iota}).
\end{equation*}
where $T_s\defeq \#\set{k: s_1^k=s}$ denotes the number of occurrences
of $s$ among all the initial states, and (i) uses that $\sum_s T_s=T$
and the Cauchy-Schwarz inequality (or
pigeonhole principle). Note that we does not know $\set{T_s}_{s\in
  \mc{S}}$ before the algorithm starts to play and thus cannot use
$\eta_+(s)= \sqrt{\log A/AT_s}$. We instead use the EXP3++
algorithm~\citep{seldin2014one} whose step-size
$\eta_{+,k}(s) = \sqrt{\log A/AN_k(s)}$ is computable at each episode
$k$. 

\paragraph{Bounding $\WeakReg_-$}
For any $\nu$ define $r(s_2, \nu(s_2))\defeq
\E_{b\sim\nu(\cdot|s_2)}[r(s_2, b)]$ for convenience.
We have
\begin{align*}
  & \quad \WeakReg_- = \sum_{k=1}^K V^{\mu^k, \nu^k}(s_1^k) -
    \min_{\nu}\sum_{k=1}^K V^{\mu^k, \nu}(s_1^k) \\
  & = \sum_{k=1}^K \E_{a\sim\mu^k(\cdot|s_1)} \brac{r(s_1^k,
    a) + \P_1[r(s_2, \nu^k(s_2))](s_1^k, a)} \\
  & \qquad - \min_{\nu}\sum_{k=1}^ K \E_{a\sim\mu^k(\cdot|s_1)}
    \brac{ r(s_1^k, a) + \P_1[r(s_2, \nu(s_2))](s_1^k, a)} \\
  & \stackrel{(i)}{=} \E_{a\sim\mu^k(\cdot|s_1)}
    \brac{ \sum_{k=1}^K r(s_1^k, a) + \P_1[r(s_2,
    \nu^k(s_2))](s_1^k, a)} \\
  & \qquad - \sum_{k=1}^ K \E_{a\sim\mu^k(\cdot|s_1)}
    \brac{ r(s_1^k, a) + \P_1[r(s_2,\nu^\star(s_2))](s_1^k, a)} \\
  & = \sum_{k=1}^K \E_{a\sim\mu^k(\cdot|s_1),s_2\sim
    \P_1(\cdot|s_1^k, a)}
  [r(s_2, \nu^k(s_2)) - r(s_2, \nu^\star(s_2))],
\end{align*}
where (i) follows from the fact that if we define $\nu^\star(s_2) =
\argmin_{b'} r(s_2, b')$, then $\nu^\star$ is optimal at every
state $s_2$ and thus also attains the minimum outside. Defining
$f_k(s_2) = r(s_2, \nu^k(s_2)) - r(s_2, \nu^\star(s_2))$, we have
that $f_k(s_2)\in[0,1]$ and is a fixed function of $s_2$ before
playing episode $k$. Thus, if we define
\begin{equation*}
  \Delta_k = \E_{a,s_2}[f_k(s_2)] - f_k(s_2^k),
\end{equation*}
then $\Delta_k$ is a bounded martingale difference sequence adapted to
$\mc{F}_{k-1}$, so by the Azuma-Hoeffding inequality
we have with probability at least $1-\delta$ that
\begin{equation*}
  \abs{\sum_{k=1}^K \Delta_k} \le C\sqrt{K\log(1/\delta)} =
  C\sqrt{T\log(1/\delta)}.
\end{equation*}
On this event, we have
\begin{align*}
  & \quad \WeakReg_- = \sum_{k=1}^K f_k(s_2^k) + \sum_{k=1}^K \Delta_k
  \\
  & \le \underbrace{\sum_{k=1}^K \brac{r(s_2^k, \nu^k(s_2)) - r(s_2,
    \nu^\star(s_2))}}_{\rm I} + C\sqrt{K\log(1/\delta)}.
\end{align*}
The first term above is the regret for the contextual bandit
problem (with context $s_2$) that the min player faces. Further, the
min player in Algorithm~\ref{algorithm:mirror-descent} plays the
mirror descent (EXP3) algorithm independently for each context $s_2$.
Therefore, by standard regret bounds for mirror descent (e.g. Theorem
11.1,~\cite{lattimore2018bandit}) we have (choosing $\eta_-\equiv\sqrt{\log
  B/T}$ in the fixed $s_2$ case, and using the EXP3++
scheduling~\cite{seldin2014one}) for the contextual case), we have 
\begin{equation*}
  {\rm I} \le \sum_{s\in\mc{S}_2} C\sqrt{BT_s\log B} \le C\sqrt{SBT\log
    B},
\end{equation*}
which combined with the above bound gives that with high probability
\begin{equation*}
  \WeakReg_- \le \cO(\sqrt{SBT\iota}),
\end{equation*}
where $\iota=\log(SABT/\delta)$.
\end{proof-of-theorem}




\section{Subroutine \explore} \label{sec:proof_explore}
In this section, we present the \explore~algorithm, as well as the proofs for Lemma \ref{lem:policy_evaluation}. The algorithm and results presented in this section is simple adaptation of the algorithm in \citet{jin2019reward}, which studies reward-free exploration in the single-agent MDP setting.

Since the guarantee of Lemma \ref{lem:policy_evaluation} only involves the evaluation of the value under fixed policies, it does not matter whether players try to maximize the reward or minimize the reward. Therefore, to prove Lemma \ref{lem:policy_evaluation} in this section, with out loss of generality, we will treat this Markov game as a single player MDP, where the agent take control of both players' actions in MG. 
For simplicity, prove for the case $\cS_1 = \cS_2 = \cdots = \cS_H$, $\cA_1 = \cA_2 = \cdots = \cA_H = \cA$. It is straightforward to extend the proofs in this section to the setting where those sets are not equal.

\begin{algorithm}[tb]
   \caption{\explore} \label{alg:rf_exp}
\begin{algorithmic}[1]
  \STATE {\bfseries Input:} iteration number $N_0$, $N$.
  \STATE set policy class $\Psi \leftarrow \emptyset$, and dataset $\cD \leftarrow \emptyset$.
  \FOR{all $(s, h) \in \cS \times [H]$} \label{line:policy_cover_start}
  \STATE $r_{h'}(s', a') \leftarrow \mathds{1}[s'=s \text{~and~} h'=h]$ for all $(s', a', h') \in \cS \times \cA \times [H]$.\label{line:reward_def}
  \STATE $\Phi^{(s,h)} \leftarrow \EULER(r, N_0)$.
  \STATE $\pi_{h}(\cdot|s) \leftarrow \text{Uniform}(\cA)$ for all $\pi \in \Phi^{(s,h)}$.
  \STATE $\Psi \leftarrow \Psi \cup \Phi^{(s,h)}$.
  \ENDFOR \label{line:policy_cover_end}
  \FOR{$n=1 \ldots N$}\label{line:sample_start}
  \STATE sample policy $\pi \sim \text{Uniform}(\Psi)$.
  \STATE play $\mathcal{M}$ using policy $\pi$, and observe the trajectory $z_n = (s_1, a_1, r_1, \ldots, s_H, a_H, r_H, s_{H+1})$.
  \STATE $\cD \leftarrow \cD \cup \{z_n\}$
  \ENDFOR \label{line:sample_end}
  \FOR{all $(s, a, h) \in \cS \times \cA \times [H]$} \label{line:empricial_start}
  \STATE $N_h(s, a) \leftarrow \sum_{(s_h, a_h) \in \cD} \mathds{1}[s_h = s, a_h = a]$.
  \STATE $R_h(s, a) \leftarrow \sum_{(s_h, a_h, r_h) \in \cD} r_h \mathds{1}[s_h = s, a_h = a]$.
  \STATE $\hat{r}_h(s, a)\leftarrow R_h(s, a)/N_h(s, a)$.
  \FOR{all $s' \in \cS$}
  \STATE $N_h(s, a, s') \leftarrow \sum_{(s_h, a_h, s_{h+1}) \in \cD} \mathds{1}[s_h = s, a_h = a, s_{h+1} = s'] $.
  \STATE $\hat{\P}_h(s'|s, a) \leftarrow N_h(s,a, s')/N_h(s, a)$.
  \ENDFOR
  \ENDFOR\label{line:empirical_end}
  \STATE {\bfseries Return:} empirical transition $\hat{\P}$, empirical reward $\hat{r}$.
\end{algorithmic}
\end{algorithm}

The algorithm is described in Algorithm \ref{alg:rf_exp}, which consists of three loops. The first loop computes a set of policies $\Psi$. By uniformly sampling policy within set $\Psi$, one is guaranteed visit all ``significant'' states with reasonable probabilities. The second loop simply collecting data from such sampling procedure for $N$ episodes. The third loop computes empirical transition and empirical reward by averaging the observation data collected in the second loop. We note Algorithm \ref{alg:rf_exp} use subroutine \EULER, which is the algorithm presented in \citet{zanette2019tighter}.

We can prove the following lemma, where Lemma \ref{lem:policy_evaluation} is a direct consequence of Lemma \ref{lem:plan}.
\begin{lemma}\label{lem:plan}
There exists absolute constant $c>0$, for any $\epsilon>0$, $p\in(0, 1)$,  if we set $N_0\ge cS^{3}AH^6 \iota^3/\epsilon$, and $N \ge cH^{5}S^2A\iota/\epsilon^2$ where $\iota \defeq \log(SAH/(p\epsilon))$, then with probability at least $1-p$, for any policy $\pi$:
\begin{equation*}
|\hat{V}^{\pi}_{1}(s_1) -V^{\pi}_{1}(s_1)| \le \epsilon/2
\end{equation*}
where $\hat{V}, V$ are the value functions of $\mdp(\hat{\P}, \hat{r})$ and $\mdp(\P, r)$, and $(\hat{\P}, \hat{r})$ is the output of the algorithm \ref{alg:rf_exp}.
\end{lemma}

\begin{proof}
The proof is almost the same as the proof of Lemma 3.6 in \citet{jin2019reward} except that there is no error in estimating $r$ in \citet{jin2019reward}. We note the error introduced by the difference of $\hat{r}$ and $r$ is a same or lower order term compared to the error introduced by the difference of $\hat{\P}$ and $\P$. We can bound the former error using the similar treatment as in bounding the latter error. This finishes the proof.
\end{proof}

\section{Connection to Algorithms against Adversarial Opponents and \textsc{R-max}}
\label{app:connection}

Similar to the standard arguments in online learning, we can use any algorithm with low regret against adversarial opponent in Markov games to design a provable self-play algorithm with low regret.

Formally, suppose algorithm $\cA$ has the following property. The max-player runs algorithm $\cA$ and has following guarantee:
\begin{equation} \label{eq:regret_guarantee}
    \max_{\mu} \sum_{k=1}^K V_1^{\mu, \nu^k}(s_1^k) -
    \sum_{k=1}^K V_1^{\mu^k, \nu^k}(s_1^k)  \le f(S, A, B, T)
\end{equation}
where $\{\mu_k\}_{k=1}^K$ are strategies played by the max-player, $\{\nu_k\}_{k=1}^K$ are the possibly adversarial strategies played by the opponent, and function $f$ is a regret bound depends on $S, A, B, T$. Then, by symmetry, we can also let min-player runs the same algorithm $\cA$ and obtain following guarantee:
\begin{equation*}
  \sum_{k=1}^K
    V_1^{\mu^k, \nu^k}(s_1^k) - \min_{\nu} \sum_{k=1}^K
    V_1^{\mu^k, \nu}(s_1^k) \le f(S, B, A, T).
\end{equation*}
This directly gives a self-play algorithm with following regret guarantee
\begin{align*}
&\WeakReg(T) = \max_{\mu} \sum_{k=1}^K V_1^{\mu,
    \nu^k}(s_1^k) 
    - \min_{\nu} \sum_{k=1}^K V_1^{\mu^k, \nu}(s_1^k) \\
     =& \max_{\mu} \sum_{k=1}^K V_1^{\mu, \nu^k}(s_1^k) -
    \sum_{k=1}^K V_1^{\mu^k, \nu^k}(s_1^k) + \sum_{k=1}^K
    V_1^{\mu^k, \nu^k}(s_1^k) - \min_{\nu} \sum_{k=1}^K
    V_1^{\mu^k, \nu}(s_1^k) \le f(S, A, B, T) + f(S, B, A, T)
\end{align*}

However, we note there are two notable cases, despite they are also results with guarantees against adversarial opponent, their regret are not in the form \eqref{eq:regret_guarantee}, thus can not be used to give self-play algorithm, and obtain regret bound in our setting.

The first case is \textsc{R-max} algorithm \cite{brafman2002r}, which studies Markov games, with guarantees in the following form.
\begin{equation*}
\sum_{k=1}^K V_1^{\mu^\star, \nu^\star}(s_1^k) - \sum_{k=1}^K V_1^{\mu^k, \nu^k}(s_1^k) \le  g(S, A, B, T)
\end{equation*}
where $\{\mu_k\}_{k=1}^K$ are strategies played by the max-player, $\{\nu_k\}_{k=1}^K$ are the adversarial strategies played by the opponent, $(\mu^\star, \nu^\star)$ are the Nash equilibrium of the Markov game, $g$ is a bound depends on $S, A, B, T$. We note this guarantee is weaker than \eqref{eq:regret_guarantee}, and thus can not be used to obtain regret bound in the setting of this paper.

The second case is algorithms designed for adversarial MDP \citep[see e.g.][]{zimin2013online, rosenberg2019online, jin2019learning}, whose adversarial opponent can pick adversarial reward function. We note in Markov games, the action of the opponent not only affects the reward received but also affects the transition to the next state. Therefore, these results for adversaril MDP with adversarial rewards do not directly apply to the setting of Markov game.

\end{document}